\documentclass[journal]{IEEEtran}
\ifCLASSINFOpdf
\else
\fi
\usepackage{microtype}
\usepackage{graphicx}
\usepackage{subfigure}
\usepackage{booktabs}

\usepackage{hyperref}

\usepackage{amsmath}
\usepackage{mathrsfs}
\usepackage{amssymb}
\usepackage{amsthm}
\usepackage{graphicx}
\usepackage{enumitem}
\usepackage{color}
\usepackage{multirow}
\usepackage{verbatim}
\usepackage{algorithm}
\usepackage{algorithmic}
\usepackage{bbm}
\usepackage{bigstrut}

\usepackage[section]{placeins}

\newcommand{\norm}[1]{\left\lVert#1\right\rVert}

\newtheorem{theorem}{Theorem}[section]

\newtheorem{lemma}[theorem]{Lemma}

\newtheorem{definition}[theorem]{Definition}

\newtheorem{assumption}[theorem]{Assumption}
\newtheorem{remark}[theorem]{Remark}

\newcommand{\M}{\mathcal{M}}

\newcommand{\A}{\mathcal{A}}

\newcommand{\T}{\mathrm{T}}

\newcommand{\R}{\mathbb{R}}

\newcommand{\etal}{ et al. }
\newcommand{\argmin}{\mathop{\rm argmin}}
\newcommand{\RR}{\mathbb{R}}

\newcommand{\Prox}{{\rm Prox}}

\newcommand{\Tr}{\mathrm{Tr}}
\newcommand{\st}{\mathrm{s.t. }}
\newcommand{\ie}{\mathrm{i.e. }}

\newcommand{\St}{\mathrm{St}}

\newcommand{\be}{\begin{equation}}
\newcommand{\ee}{\end{equation}}
\newcommand{\ba}{\begin{array}}
\newcommand{\ea}{\end{array}}
\newcommand{\bad}{\begin{aligned}}
\newcommand{\ead}{\end{aligned}}

\newcommand{\Retr}{\mathrm{Retr}}

\newcommand{\rank}{\mathrm{rank}}
\newcommand{\br}{\mathbb{R}}
\newcommand{\cP}{\mathcal{P}}
\newcommand{\Gr}{\mathrm{Gr}}
\newcommand{\myspan}{\mathrm{span}}
\newcommand{\qf}{\mathrm{qf}}

\newcommand{\grad}{\mathrm{grad}}
\newcommand{\vvec}{\mathrm{vec}}

\newcommand{\Proj}{\mathrm{Proj}}
\newcommand{\diag}{\mathrm{diag}}

\begin{document}
%
\title{Robust Low-rank Matrix Completion via an Alternating Manifold Proximal Gradient Continuation Method}
%
%
%

\author{Minhui~Huang,
        Shiqian~Ma,
        and~Lifeng~Lai
\thanks{M. Huang and L. Lai are with the Department
of Electrical and Computer Engineering, University of California, Davis, 
CA, 95616. Email: \{mhhuang, lflai\}@ucdavis.edu.}
\thanks{S. Ma is with the Department
of Mathematics, University of California, Davis, 
CA, 95616. Email: sqma@ucdavis.edu.}

\thanks{This research was partially supported by NSF HDR TRIPODS grant CCF-1934568, NSF grants CCF-1717943, CNS-1824553, CCF-1908258, DMS-1953210 and CCF-2007797, and UC Davis CeDAR (Center for Data Science and Artificial Intelligence Research) Innovative Data Science Seed Funding Program.}
}

\maketitle

\begin{abstract}
Robust low-rank matrix completion (RMC), or robust principal component analysis with partially observed data, has been studied extensively for computer vision, signal processing and machine learning applications. This problem aims to decompose a partially observed matrix into the superposition of a low-rank matrix and a sparse matrix, where the sparse matrix captures the grossly corrupted entries of the matrix. A widely used approach to tackle RMC is to consider a convex formulation, which minimizes the nuclear norm of the low-rank matrix (to promote low-rankness) and the $\ell_1$ norm of the sparse matrix (to promote sparsity). In this paper, motivated by some recent works on low-rank matrix completion and Riemannian optimization, we formulate this problem as a nonsmooth Riemannian optimization problem over Grassmann manifold. This new formulation is scalable because the low-rank matrix is factorized to the multiplication of two much smaller matrices. We then propose an alternating manifold proximal gradient continuation (AManPGC) method to solve the proposed new formulation. Convergence rate of the proposed algorithm is rigorously analyzed. Numerical results on both synthetic data and real data on background extraction from surveillance videos are reported to demonstrate the advantages of the proposed new formulation and algorithm over several popular existing approaches.
\end{abstract}


%
\IEEEpeerreviewmaketitle

\section{Introduction}
\IEEEPARstart{R}{obust} matrix completion (RMC) targets at fulfilling the missing entries of a partially observed matrix with the presence of sparse noisy entries. The recovery  relies on a critical low-rank assumption of real datasets: a low dimensional subspace can capture most information of high-dimensional observations. Therefore, the RMC model frequently arises in a wide range of applications including recommendation systems \cite{koren2009matrix}, face recognition \cite{wright2009robust}, collaborative filtering \cite{koren2008factorization}, and MRI image processing \cite{otazo2015low}.

Due to its low-rank property, matrix completion can be naturally formulated as a rank constrained optimization problem. However, it is computationally intractable to directly minimize the rank function since the low-rank constraint is discrete and nonconvex. To resolve this issue, there is a line of research focusing on convex relaxation of RMC formulations \cite{cambier2016robust, Ma-PIEEE-2018}. For example, a widely used technique in these works is employing its tightest convex proxy, $\ie$ the nuclear norm. However, calculating the nuclear norm in each iteration can bring numerical difficulty for large-scale problems in practice, because algorithms for solving them need to compute a singular value decomposition (SVD) in each iteration, which can be very time consuming. Recently, nonconvex RMC formulations based on low-rank matrix factorization were proposed in the literature \cite{guo2018low, cao2016robust, chen2015fast, shen2014augmented}. In these formulations, the target matrix is factorized as the product of two much smaller matrices so that the low-rank property is automatically satisfied. The idea of factorization greatly reduces the price of promoting low-rankness, thus relieves the pressure of computation. On the other hand, researchers have also proposed the RMC problem over a fixed rank manifold \cite{cambier2016robust} and solve it by manifold optimization. The retraction operation of a fixed rank manifold only requires performing a truncated SVD, thus is much more computationally efficient.

In many practical scenarios, the collected matrix datasets always come with noise. Without an effective approach to deal with these noise, recovering the ground truth of a low-rank matrix can be prevented from a reasonable solution. Fortunately, the noises in many real datasets have some common structures such as the sparsity. Therefore, it is reasonable to assume the noisy entries are sparse in matrix completion problems. To deal with the sparse noisy entries in matrix completion, researchers have employed the $\ell_1$-norm instead of the $\ell_0$-norm for decomposing sparse noisy entries in matrix completion formulations. This leads to the robust matrix completion model, which is robust to sparse outliers. In practice, the $\ell_1$-norm minimization can be efficiently solved by iterative soft thresholding in many cases. Some other robust loss functions such as the Huber loss were also proposed in the literature \cite{zhao2016efficient}. 

In this paper, 
by utilizing the recent developments on manifold optimization, we propose a nonsmooth RMC formulation over the Grassmann manifold with a properly designed regularizer. The Grassmann manifold automatically restricts our factorizer on a fixed dimension subspace, thus promotes the fixed low rank property. In each iteration, we only require a QR decomposition as our retraction for the Grassmann manifold, which is more computationally efficient compared with the truncated SVD. Notice that in \cite{li2018nonconvex}, the author designed a regularizer to balance the scale of two factorizers. In our formulation, the QR decomposition automatically balances the scale of two factorizers, which prevents our formulation from being ill-conditioned. Moreover, compared with previous work on matrix completion over Grassmann manifold \cite{Dai-Grassmann-MC-1,Dai-Grassmann-MC-2, Keshavan-2009,Keshavan-2009-2}, our formulation is continuous with a smaller searching space. We then propose an alternating manifold proximal gradient (AManPG) algorithm for solving it. The AManPG algorithm alternately updates between the low-rank factorizer with a Riemannian gradient step and the nonsmooth sparse variable with a proximal gradient step. While recent ADMM based methods \cite{shen2014augmented, he2011online,He-Balzano-2012} lack any convergence guarantee, we can rigorously analyze the convergence of AManPG algorithm and prove a complexity bound of $\mathcal{O}(\epsilon^{-2})$ to reach an $\epsilon$-stationary point. Furthermore, compared with recent smoothing technique developed in \cite{cambier2016robust}, we solve the nonsmooth subproblem directly.
To further boost the convergence speed and recover the low-rank matrix with a high accuracy, we apply a continuation framework \cite{hale2008fixed} to AManPG. Finally, we conduct extensive numerical experiments and show that the proposed AManPG with continuation is more efficient compared with previous works on the RMC problem.

The remainder of the paper is organized as follows. In Section \ref{sec:formulation}, we briefly review some existing related work on the RMC problem and give our new RMC formulation over Grassmann manifold. In Section \ref{sec:manpg}, we review some basics of manifold optimization and propose ManPG algorithm for our RMC formulation. In Section \ref{sec:amanpgc} and Section \ref{sec:convergence}, we propose our AManPG with continuation algorithm and provide its convergence analysis. 
In Section \ref{sec:numerical}, numerical results are presented to demonstrate the the advantages of the proposed new formulation and algorithm. Finally, the conclusions are given in Section \ref{sec:conclution}.

\section{Problem Formulation} \label{sec:formulation}
In this section, to motivate our proposed problem formulation, we will first introduce some closely related work. We then present our problem formulation.
\subsection{Related Work}

Robust principal component analysis (PCA) \cite{candes2011robust,chandrasekaran2011rank} is an important tool in data analysis and has found many interesting applications in computer vision, signal processing, machine learning, and statistics, and so on. The goal is to decompose a given matrix $M\in\mathbb{R}^{m\times n}$ into the superposition of a low-rank matrix $L$ and a sparse matrix $S$, i.e., $M=L+S$. 
The works by Cand\`es \etal \cite{candes2011robust} and Chandrasekaran \etal \cite{chandrasekaran2011rank} formulate the problem as the following convex optimization problem:
\be\label{RPCA}
\ba{lll}
\min\limits_{L,S} \ \|L\|_* + \gamma\|S\|_1, \ \st \ L+S=M,
\ea
\ee
where $\gamma>0$ is a weighting parameter, the nuclear norm $\|L\|_*$ sums the singular values of $L$, and the $\ell_1$ norm $\|S\|_1$ sums the absolute values of all entries of $S$. When only a subset of the entries of $M$ is observed, robust PCA becomes the robust low-rank matrix completion problem \cite{cambier2016robust}. Similar to \eqref{RPCA}, a convex formulation for RMC can be cast as follows:
\be\label{RMC}
\ba{lll}
\min\limits_{L,S} \ \|L\|_* + \gamma\|S\|_1, \ \st \ \cP_\Omega(L+S)=\cP_\Omega(M),
\ea
\ee
where $\Omega$ denotes the set of the indices of the observed entries and $\cP_\Omega: \br^{m\times n}\mapsto\br^{m\times n}$ denotes a projection defined as: $[\cP_\Omega(Z)]_{ij}=Z_{ij}$, if $(i,j)\in\Omega$, and $[\cP_\Omega(Z)]_{ij}=0$, otherwise.
The convex formulations \eqref{RPCA} and \eqref{RMC} have been studied extensively in the literature, and we refer to the recent survey paper \cite{Ma-PIEEE-2018} for algorithms for solving them.

 By assuming that the rank of $L$ is known (denoted by $r$), the idea of many nonconvex formulations is based on the fact that $L$ can be factorized to $L = UV$, where $U\in\br^{m\times r}$, $V\in\br^{r\times n}$. Replacing $L$ by $UV$ in \eqref{RPCA} and \eqref{RMC} leads to various nonconvex formulations for robust PCA and RMC. In particular, Li \etal \cite{li2018nonconvex} suggested using subgradient method to solve the following nonsmooth robust matrix recovery model
\be\label{SUBGM}
\ba{lll}
\min\limits_{U \in \R^{m\times r}, V \in \R^{r\times n}} \frac{1}{\lvert \Omega \rvert}\norm{y - \A(UV)}_1, \\
\ea
\ee
where $y$ is a small number of linear measurements and $\A: \RR^{m\times n} \to \RR^{\lvert \Omega \rvert}$ is a known linear operator. 
 Shen \etal \cite{shen2014augmented}  proposed the LMaFit algorithm that implements an alternating direction method of multipliers (ADMM) for solving the following nonconvex formulation of robust PCA:
\be\label{LMaFit}
\ba{lll}
\min\limits_{U \in \R^{m\times r}, V \in \R^{r\times n}, Z \in \R^{m\times n} } &\norm{\cP_{\Omega}(Z - M)}_1, \\
\hspace{15mm}\st &\ UV-Z = 0.
\ea
\ee
Note that if $(\hat{U},\hat{V})$ solves \eqref{LMaFit}, then $(\hat{U}Q, Q^{-1}\hat{V})$ also solves \eqref{LMaFit} for any invertible $Q\in\br^{r\times r}$. Since all matrices $\hat{U}Q$ share the same column space, Dai \etal \cite{Dai-Grassmann-MC-1,Dai-Grassmann-MC-2} exploited this fact and formulated the matrix completion problem as the following optimization problem over a Grassmann manifold:
\be\label{Dai-MC}
\ba{lll}
\min\limits_{U\in\Gr(m,r),V\in \R^{r\times n}} \|\cP_\Omega(UV-M)\|_F^2,
\ea
\ee
where $\Gr(m,r)$ denotes the Grassmann manifold. However, it is noticed that the outer problem for $U$ might be discontinuous at points $U$ for which the $V$ problem does not have a unique solution. To address this issue, Keshavan \etal \cite{Keshavan-2009,Keshavan-2009-2} proposed to optimize both the column space and row space at the same time, which results in the following so-called OptSpace formulation for matrix completion:
\be\label{OptSpace}
\ba{lll}
\min\limits_{U\in\Gr(m,r),V\in \Gr(n, r)}\min\limits_{\Sigma \in\br^{r\times r}} &\|\cP_\Omega(U\Sigma V^\top-M)\|_F^2 \\
&+ \lambda\|U\Sigma V^\top\|_F^2.
\ea
\ee
Here $\lambda>0$ is a weighting parameter, and the regularizer $\|U\Sigma V^\top\|_F^2$ is used so that the outer problem is continuous. Boumal and Absil \cite{RTRMC-2011,boumal2015rtrmcextended} proposed to study the following variant of \eqref{Dai-MC}:
\be\label{RTRMC}
\ba{lll}
\min\limits_{U\in\Gr(m,r),V\in \R^{r\times n}} \frac{1}{2}\|\cP_\Omega(UV-M)\|_F^2 + \frac{\lambda^2}{2}\|\cP_{\bar{\Omega}}(UV)\|_F^2,
\ea
\ee
where $\bar{\Omega}$ is the complement of $\Omega$, and they proposed to use Riemannian trust region method to solve this problem. Comparing with OptSpace \eqref{OptSpace}, formulation \eqref{RTRMC} has a much smaller searching space. Note that in \eqref{RTRMC}, $\lambda$ is usually chosen to be very close to zero, as it indicates that we have a small confidence that the entries $(UV)_{ij}$ for $(i,j)\notin\Omega$ are equal to zero.

For RMC, Cambier and Absil \cite{cambier2016robust} proposed the following Riemannian optimization formulation:
\be\label{Cambier-Absil-RMC}
\ba{lll}
\min\limits_{X\in \M_r} \norm{\cP_{\Omega}(X - M)}_1 + \lambda\|\cP_{\bar{\Omega}}(X)\|_F^2,
\ea
\ee
where $\M_r$ denotes the fixed-rank manifold, i.e., $\M_r:=\{X\mid \rank(X)=r\}$. The algorithm proposed in \cite{cambier2016robust} needs to smooth the $\ell_1$ norm first to change the problem to a smooth problem, and then applies the Riemannian conjugate gradient method to solve the smoothed problem. As a result, the algorithm in \cite{cambier2016robust} does not solve \eqref{Cambier-Absil-RMC} exactly. Related to \eqref{Dai-MC} and \eqref{LMaFit}, He \etal proposed the GRASTA algorithm \cite{he2011online,He-Balzano-2012} which can be used to solve the following formulation of RMC: 
\be\label{GRASTA}
\ba{lll}
\min\limits_{U\in\Gr(m,r),V\in\br^{r\times n}} \ \|\cP_\Omega(UV-M)\|_1.
\ea
\ee
GRASTA uses alternating minimization and ADMM to solve \eqref{GRASTA}, which is efficient in practice but lacks convergence guarantees. Moreover, Zhao \etal \cite{zhao2016efficient} explores the effect of different robust loss functions for proposing the robustness against specific categories of outliers. He \etal \cite{he2019robust} derived a correntropy-based cost function and applied the half-quadratic technique to solve the formulation. Zeng \etal \cite{zeng2017outlier} proposed two schemes, namely the iterative $\ell_p$-regression algorithm and ADMM for RMC under $\ell _p$ minimization. Zhang \etal \cite{zhang2019correction} proposed the RMC problem over Hankel matrix and claimed that they can deal with case when all the observations in one column are erroneous. 

\subsection{Our formulation and contributions} \label{sec:contribution}

Motivated by these existing works, in this paper, we propose to solve the following formulation of RMC:
\be\label{ourRMC}
\ba{lll}
\min\limits_{U \in \Gr(m,r), V\in\br^{r\times n}, S\in \mathbb{R}^{m \times n}} F(U,V,S) =\\
\hspace{-4mm}\frac{1}{2}\norm{ \mathcal{P}_{\Omega}(UV - M + S)}_F^2 + \frac{\lambda^2}{2}\norm{\mathcal{P}_{\bar{\Omega}}(UV)}_F^2 + \gamma\norm{ \mathcal{P}_{\Omega}(S)}_1.
\ea
\ee
We show that the manifold proximal gradient method (ManPG) proposed by Chen \etal \cite{chen2019proximal} can be applied to solve \eqref{ourRMC} and the corresponding convergence analysis applies naturally. We then propose a variant of ManPG, named alternating ManPG (AManPG) that can significantly improve the efficiency of ManPG for solving \eqref{ourRMC}. We further rigorously analyze the convergence rate of AManPG. Compared with GRASTA, our proposed algorithms for solving \eqref{ourRMC} have rigorous convergence guarantees and convergence rate analysis. Compared with RMC \eqref{Cambier-Absil-RMC}, our algorithms solve the nonsmooth problem \eqref{ourRMC} directly. Compared with the convex formulation \eqref{RMC}, our nonconvex formulation appears to be more robust and scalable, see the numerical experiments for comparison results. Finally, to further accelerate the convergence of AManPG, we incorporate the so-called continuation technique on the weighting parameter $\gamma$ in \eqref{ourRMC}. Our numerical results on both synthetic data and real data on background extraction from surveillance video demonstrate that our final algorithm, AManPG with Continuation (AManPGC), compares favorably with existing methods for RMC.


\section{The ManPG Algorithm for Robust Matrix Completion}\label{sec:manpg}

In this section, we show that the ManPG algorithm recently proposed by Chen \etal \cite{chen2019proximal} can be naturally adopted to solve \eqref{ourRMC}. Note that the ManPG algorithm was originally proposed for solving problems over the Stiefel manifold, but the manifold in \eqref{ourRMC} is Grassmann manifold. Therefore, we need to further elaborate on the details how ManPG works on Grassmann manifold. To this end, we first introduce some backgrounds on the geometry of Grassmann manifold.

\subsection{Geometry of the Grassmann Manifold}

In this subsection we briefly introduce concepts and properties of Grassmann manifold. Much of the materials here are from \cite{RTRMC-2011}, and we include them here for the ease of discussion later. Grassmann manifold $\Gr(m,r)$ is the set of $r$-dimensional linear subspaces of $\R^m$ endowed with quotient manifold structure, whose dimension is $\dim(\Gr(m,r)) = r(m - r)$ \cite{absil2009optimization}. Each point of $\Gr(m,r)$ is a linear subspace spanned by the column space of a full-rank matrix $U$:
\be\label{def-Gr}
\ba{lll}
\Gr(m,r) = \{\myspan(U): U \in \R_*^{m\times r}\},
\ea
\ee
where $\R_*^{m\times r}$ denotes the set of all $m \times r$ matrices with full column rank, and $\myspan(U)$ denotes the subspace spanned by the columns of $U$. Since multiplying by an $r \times r$ orthonormal matrix does not change the column space of $U$, we can regard $\Gr(m,r)$ as a quotient of $\R_*^{m\times r}$ by the equivalent relation $U' = UQ$, where $Q$ is any $r \times r$ orthonormal matrices. Endowed with the Riemannian metric $\langle U, V \rangle = \Tr(UV)$, the Grassmann manifold is also a Riemannian quotient manifold, and it admits a tangent space at each point of $\Gr(m,r)$ given by
\be\label{grass-tangent-space}
\ba{lll}
\T_U\Gr(m,r) = \{H \in \R^{m\times r}: U^\top H = 0 \}.
\ea
\ee
For Riemannian manifold $\M$, the Riemannian gradient of a smooth function $f: \M \to \R$ is defined as follows.
\begin{definition} {(Riemannian Gradient)}
Given a smooth function $f: \M \to \R$, the Riemannian gradient of $f$  at $X \in \M$, denoted by $\grad f(X)$,  is the unique tangent vector in $\T_X\M$ such that
\be\label{def-riemannain-grad}
\langle \grad f(X), \xi \rangle = Df(X)[\xi], \quad \forall \xi \in \T_X\M,
\ee
where $Df$ denotes the directional derivatives of $f$.
\end{definition}
For Grassmann manifold, the orthogonal projector from $\R^{m \times r}$ onto the tangent space $\T_U\Gr(m,r)$ is given by:
\be\label{projector}
\ba{lll}
\Proj_U: \R^{m \times r} \to \T_U \Gr(m,r),\\
\Proj_U(H) = (I - UU^\top)H.
\ea
\ee
The definition of retraction operation for manifold $\M$ is given below.
\begin{definition} {(Retraction)}
Let $\Retr_X(\xi):\T\mathcal{M} \to \mathcal{M}$ be a mapping from the tangent bundle $\T\mathcal{M}$ to the manifold $\mathcal{M}$. We call $\Retr_X(\cdot)$ a retraction at $X$ if
\be\label{retraction}
\ba{lll}
\Retr_X(0) = X, \quad \frac{d}{dt}\Retr_X(t\xi)|_{t=0} = \xi,\\
\forall X\in \mathcal{M},\quad \forall \xi \in \T_X\mathcal{M}.
\ea
\ee
\end{definition}

In our numerical experiments, we choose the QR decomposition as the retraction for Grassmann manifold:
\be\label{qr}
\Retr_U(H) = \qf(U+H),
\ee
where $\qf(X)$ denotes the $Q$-factor of the QR decomposition of $X$.

\subsection{The ManPG Algorithm}

Recently, Chen \etal \cite{chen2019proximal} proposed a novel ManPG algorithm for solving nonsmooth optimization problem over the Stiefel manifold $\St(n,r)$ in the following form:
\be\label{manpg-composite-form}
\min\limits_{X \in\St(n,r)} F_1(X) + F_2(X),
\ee
in which $F_1$ is smooth with Lipschitz continuous gradient, $F_2$ is nonsmooth and convex. Here the smoothness, convexity and Lipschitz continuity are interpreted when the functions are considered in the ambient Euclidean space. A typical iteration of ManPG algorithm for solving \eqref{manpg-composite-form} is as follows:
\be\label{ManPG-general}\ba{lll}
Y^k     & := & \argmin_Y \ \langle \nabla F_1(X^k), Y \rangle + \frac{1}{2t}\|Y\|_F^2 \\
&&+ F_2(X^k+Y), \ \st, \ Y\in\T_{X^k}\St(n,r) \\
X^{k+1} & := & \Retr_{X^k}(\alpha_kY^k),
\ea\ee
where $t>0$ and $\alpha_k>0$ are step sizes. Chen \etal \cite{chen2019proximal} proved that the iteration complexity of ManPG \eqref{ManPG-general} is $O(1/\epsilon^2)$ for obtaining an $\epsilon$-stationary solution. They also demonstrated that ManPG is very efficient for solving sparse PCA and compressed modes problems.

Now we discuss how to apply ManPG \cite{chen2019proximal} to solve \eqref{ourRMC}. For ease of presentation, we denote the smooth part of $F$ in \eqref{ourRMC} by $\bar{f}$, i.e.,
\be\label{def-f-bar}
\ba{lll}
\hspace{-2mm}\bar{f}(U,V,S) = \frac{1}{2}\norm{\mathcal{P}_{\Omega}(UV - M + S)}_F^2 + \frac{\lambda^2}{2}\norm{\mathcal{P}_{\bar{\Omega}}(UV)}_F^2.
\ea
\ee
Note that for fixed $U$ and $S$, the optimal $V$ of $F(U,V,S)$ (and $\bar{f}(U,V,S)$) is uniquely determined. Therefore, by denoting
\be\label{def-V}
\ba{lll}
V_{U,S}:=\argmin_V \ \bar{f}(U,V,S), 
\ea
\ee
and 
\be\label{def-f}
\ba{lll}
 f(U,S) = \bar{f}(U,V_{U,S},S),
\ea
\ee
we know that our RMC formulation \eqref{ourRMC} reduces to
\be\label{ourRMC-noV}
\min\limits_{U\in\Gr(m,r),S\in\br^{m\times n}} \ f(U,S) + \gamma\norm{ \mathcal{P}_{\Omega}(S)}_1.
\ee
It is easy to see that ManPG for solving \eqref{ourRMC-noV} reduces to the following two subproblems in the $k$-th iteration:
\begin{subequations}\label{ManPG-RMC}
\begin{align}
\begin{split}
\Delta S^{k} :=&  \quad \argmin_{\Delta S} \ \langle \nabla_S f(U^k,S^k), \Delta S \rangle + \frac{1}{2t_S}\|\Delta S\|_F^2 \\
&\quad + \gamma\norm{\mathcal{P}_{\Omega}(S^k+\Delta S)}_1 \label{ManPG-RMC-dS}\\
\end{split}\\
\begin{split}
\Delta U^{k} := & \quad \argmin_{\Delta U} \ \langle \nabla_U f(U^k,S^k), \Delta U \rangle + \frac{1}{2t_U}\|\Delta U\|_F^2,\\
& \quad\ \st, \ \Delta U \in \T_{U^k}\Gr(m,r) \label{ManPG-RMC-dU}\\
\end{split}\\
\begin{split}
S^{k+1} & := \quad S^k + \alpha\Delta S^k,\label{ManPG-RMC-S+}\\
\end{split}\\
\begin{split}
U^{k+1} & := \quad \Retr_{U^k}(U^k + \beta \Delta U^k) \label{ManPG-RMC-U+},\\
\end{split}
\end{align}
\end{subequations}
where $t_S$, $t_U$, $\alpha$ and $\beta$ are all step sizes. We now make some necessary remarks on this ManPG algorithm \eqref{ManPG-RMC}. First, \eqref{ManPG-RMC} is actually slightly different with a direct application of ManPG for solving \eqref{ourRMC-noV}. For a direct application of ManPG, we should have $t_S=t_U$ and $\alpha=\beta$. Here in \eqref{ManPG-RMC} we allow these step sizes to be different so that we have more freedom to choose the best step sizes in practice. Also we note that \eqref{ManPG-RMC-dU} and \eqref{ManPG-RMC-U+} correspond to a Riemannian gradient step with respect to the $U$ variable. The updates \eqref{ManPG-RMC-dS} and \eqref{ManPG-RMC-S+} correspond to a proximal gradient step for the $S$ variable in the Euclidean space. This can be interpreted in the following way. First, in RMC \eqref{ourRMC-noV}, the $U$ variable does not appear in the nonsmooth part of the objective, so it is natural to perform a Riemannian gradient step for $U$. Second, for fixed $U$, the $S$ problem is only an unconstrained problem in the Euclidean space, so it is reasonable to take a proximal gradient step. Moreover, the two subproblems \eqref{ManPG-RMC-dU} and \eqref{ManPG-RMC-dS} are very easy to solve. 

\begin{algorithm}[tbp]
\caption{ManPG for solving RMC \eqref{ourRMC-noV}}\label{alg:manpg}
\begin{algorithmic}[1]
\STATE{Input: step sizes $t_S$, $t_U$,$\alpha$, $\beta$, parameters $\lambda$, $\gamma$, accuracy tolerance $\epsilon$, and initial point $(U^0,S^0)$. }
\FOR{$k = 0, 1, \ldots $}
\STATE{Compute $V^k_{U^k, S^k}$ using \eqref{Compute-V}}
\STATE{Compute $\Delta S^k$ by \eqref{ManPG-RMC-dS-simple}}
\STATE{Update  $S^{k+1}$ by \eqref{ManPG-RMC-S+}}
\STATE{Compute $\Delta U^k$ by \eqref{ManPG-RMC-dU-simple}} 
\STATE{Update  $U^{k+1}$ by \eqref{ManPG-RMC-U+}}
\IF{$\norm{\Delta U^{k+1}}_F^2 + \norm{\Delta S^{k+1}}_F^2 \leq\epsilon^2$}
\STATE{break}
\ENDIF
\ENDFOR
\STATE{Output: $U^{k+1}$, $S^{k+1}$}
\end{algorithmic}
\end{algorithm}

Specifically, \eqref{ManPG-RMC-dU} can be reduced to
\be\label{ManPG-RMC-dU-simple}
\Delta U^{k} = -t_U\grad_U f(U^k,S^k),
\ee
i.e., it is the negative Riemannian gradient of $f$ multiplied by the step size $t_U$. The $\Delta S$ subproblem \eqref{ManPG-RMC-dS} can be solved by a simple $\ell_1$ norm shrinkage operation (note that we are only interested in the $\cP_\Omega(\Delta S^k)$ and $\cP_{\bar{\Omega}}(\Delta S^k)$ can be simply set to $0$):
\be\label{ManPG-RMC-dS-simple}
\ba{lll}
\cP_\Omega(\Delta S^k) = &\Prox_{\gamma t_S\|\cdot\|_1}(\cP_\Omega(S^k-t_S\nabla_S f(U^k,S^k)))\\
& - \cP_\Omega(S^k).
\ea
\ee
Now, to implement the ManPG \eqref{ManPG-RMC}, the only remaining component is to calculate the Riemannian gradient $\grad_U f(U,S)$ used in \eqref{ManPG-RMC-dU-simple}. The procedure for computing it is outlined in \cite{boumal2015rtrmcextended}. By assuming that the subspace of the Grassmann manifold is represented by orthonormal bases, which means $U$ is restricted to the Stiefel manifold, the Riemannian gradient of the smooth function $f(U,S)$ with respect to $U$ is given by:
\be\label{grad_Uf}
\ba{lll}
\grad_U f(U,S) = &((1 - \lambda^2)C \odot (UV_{U,S} - M + S) -\\
 &\lambda^2(M-S))V_{U,S}^\top + \lambda^2 U(V_{U,S}V_{U,S}^\top),
\ea
\ee
where $C\in\br^{m\times n}$ is the mask operator whose components are given by: $C_{ij}=1$ if $(i,j)\in\Omega$, and $C_{ij}=0$ otherwise. Here $V_{U,S}$ is defined in \eqref{def-V}, and it can be computed as follows:
\begin{equation} \label{Compute-V}
\vvec(V_{U,S}) = A^{-1} \vvec(U^\top [C \odot (M - S)]),
\end{equation}
where $\vvec$ denotes the vectorization operator, and $A$ is defined as
\[
A = (I_n \otimes U^\top)\diag(\vvec((1-\lambda^2)C))(I_n \otimes U) + \lambda^2 I_{rn}.
\]
and $\otimes$ denotes the Kronecker product. For more details about these calculation, we refer the reader to \cite{boumal2015rtrmcextended}.

With these preparations, we can finally summarize the ManPG algorithm \eqref{ManPG-RMC} for solving \eqref{ourRMC} (or, \eqref{ourRMC-noV}) as in Algorithm \ref{alg:manpg}.

\section{Alternating ManPG with Continuation}\label{sec:amanpgc}

It should be noted that ManPG updates $S$ and $U$ in parallel. That is, ManPG \eqref{ManPG-RMC} is a Jacobi type iterative algorithm. One way that can possibly improve the speed of ManPG is to use a Gauss-Seidel type algorithm. This idea has also been adopted in \cite{Chen-AManPG-2019}, where the authors showed that the Gauss-Seidel type ManPG performs much better than the original Jacobi type ManPG. Motivated by this, here we also propose an alternating ManPG (AManPG), which updates $S$ and $U$ sequentially, instead of in parallel. However, one crucial thing to note here is that the $V$ variable will need to be re-calculated, when we have a new $S$ variable before we update the $U$ variable. Our AManPG algorithm is summarized in Algorithm \ref{alg:Amanpg}.


\begin{algorithm}[ht]
\caption{AManPG for solving \eqref{ourRMC-noV}}\label{alg:Amanpg}
\begin{algorithmic}[1]
\STATE{Input: step sizes $t_S$, $t_U$,$\alpha$, $\beta$, parameters $\lambda$, $\gamma$, threshold $\epsilon$, $U^0, S^0$. }
\FOR{$k = 0,1,\ldots$}
\STATE{Compute $V^k_{U^k, S^k}$ using \eqref{Compute-V}}
\STATE{Compute $\Delta S^k$ by \eqref{ManPG-RMC-dS-simple}}
\STATE{Update $S^{k+1}$ by \eqref{ManPG-RMC-S+}}
\STATE{Compute $V^k_{U^k, S^{k+1}}$ using \eqref{Compute-V}}
\STATE{Compute $\Delta U^k =-t_U\grad_U f(U^k,S^{k+1})$} 
\STATE{Update $U^{k+1}$ by \eqref{ManPG-RMC-U+}}
\IF{$\norm{\Delta U^{k+1}}_F^2 + \norm{\Delta S^{k+1}}_F^2 \leq \epsilon^2$}
\STATE{break}
\ENDIF
\ENDFOR
\STATE{Output: $U^{k+1}$, $S^{k+1}$}
\end{algorithmic}
\end{algorithm}

\begin{remark}
When we compute $\Delta U^k$ in AManPG, we used the latest $S^{k+1}$, which requires us to compute the latest $V^k_{U^k, S^{k+1}}$. While in ManPG, we used $S^k$ in the updates of $\Delta U^k$, and this does not require us to compute another $V^k$. This is the main difference between AManPG (Algorithm \ref{alg:Amanpg}) and ManPG (Algorithm \ref{alg:manpg}). In both algorithms, we always set $\alpha=\beta=1$. Noticing that the matrix $A$ only depends on variable $U$, there is no need to recalculate $A$ when computing $V^k_{U^k, S^{k+1}}$.
\end{remark}

\begin{algorithm}[h]
\caption{AManPG with Continuation (AManPGC) for solving \eqref{ourRMC-noV}}\label{alg:AmanpgC}
\begin{algorithmic}[1]
\STATE{Input: Step sizes $t_S$, $t_U$, parameters $\gamma_0\gg\gamma_{\min}$, shrinking factors $\mu_1<1$, $\mu_2<1$, initial accuracy tolerance $\epsilon^0$}
\STATE{Initialize: $U^0$, $S^0$. Set $\ell = 0$}
\WHILE{$\gamma_\ell > \gamma_{\min}$}
\STATE{Call AManPG to solve \eqref{ourRMC-noV} with $\gamma=\gamma_\ell$, and set the output of AManPG as $(U^{\ell+1},S^{\ell+1})$.}
\STATE{$\gamma_{\ell+1} = \mu_1\gamma_\ell$}
\STATE{$\epsilon_{\ell+1} = \mu_2\epsilon_\ell$}
\STATE{$\ell = \ell + 1$}
\ENDWHILE
\STATE{Output: $U^\ell$, $S^\ell$}
\end{algorithmic}
\end{algorithm}

{\bf The continuation technique.} There are two parameters in the model \eqref{ourRMC}: $\lambda$ and $\gamma$. Since $\lambda$ indicates our confidence level of the entries of $(UV)$ being zero, it needs to very small. In practice, it is easy to choose $\lambda$, and in our numerical experiments, we always choose $\lambda = 10^{-8}$. The parameter $\gamma$ in \eqref{ourRMC} controls the sparsity level of $\cP_\Omega(S)$. A larger $\gamma$ yields sparser $\cP_\Omega(S)$. However, in practice, we usually have no clue how sparse the matrix $S$ should be. Thus, it is not easy to choose $\gamma$. A usual practice in the literature to deal with this issue is to conduct a continuation technique on $\gamma$. Roughly speaking, the continuation starts with solving \eqref{ourRMC} with a relatively large $\gamma$. Then the parameter $\gamma$ is decreased and \eqref{ourRMC} is solved again. This process is repeated until $\gamma$ is very small. This idea has been widely adopted in the literature, e.g., \cite{hale2008fixed,ma2011fixed,goldfarb2011convergence,Xiao-Zhang-homotopy-2013}. Combining this continuation idea with our AManPG algorithm, we obtain the AManPGC algorithm which works greatly in practice as confirmed by our numerical results in Section \ref{sec:numerical}. AManPGC is summarized in Algorithm \ref{alg:AmanpgC}. Note that in Algorithm \ref{alg:AmanpgC}, we also shrink the accuracy tolerance $\epsilon$ in each iteration, as we want to solve the problem more and more accurately.

\section{Convergence Analysis for AManPG}\label{sec:convergence}

In this section, we analyze the convergence behavior and iteration complexity of AManPG (Algorithm \ref{alg:Amanpg}). To simplify the notation, we denote $\M=\Gr(m,r)$ and $h(S)=\gamma\|\cP_\Omega(S)\|_1$, and we analyze the convergence of AManPG for solving the following problem:
\be\label{ourRMC-noV-rewrite}
\min\ F(U, S) = f(U, S) + h(S), \ \st, \ U\in\M,
\ee
where $f(U,S)$ is smooth and $h(S)$ is nonsmooth and convex. Here the smoothness and convexity are interpreted when the functions are considered in the ambient Euclidean space. For simplicity, we rewrite the AManPG for solving \eqref{ourRMC-noV-rewrite} here. One typical iteration of AManPG for solving \eqref{ourRMC-noV-rewrite} is:
\begin{subequations}\label{AManPG-general}
\begin{align}
\begin{split}
\Delta S^{k} := & \quad \argmin_{\Delta S} \ \langle \nabla_S f(U^k,S^k), \Delta S \rangle + \frac{1}{2t_S}\|\Delta S\|_F^2\\
& \quad + h(S^k+\Delta S), \label{AManPG-general-dS}\\
\end{split}\\
\begin{split}
S^{k+1} := & \quad S^k + \alpha \Delta S^k,\label{AManPG-general-S+}\\
\end{split}\\
\begin{split}
\Delta U^{k} := & \quad \argmin_{\Delta U} \ \langle \nabla_U f(U^k,S^{k+1}), \Delta U \rangle + \frac{1}{2t_U}\|\Delta U\|_F^2,\\
& \quad \ \st \ \Delta U \in \T_{U^k}\M \label{AManPG-general-dU}\\
\end{split}\\
\begin{split}
U^{k+1} := & \quad \Retr_{U^k}(U^k + \beta\Delta U^k) \label{AManPG-general-U+}.\\
\end{split}
\end{align}
\end{subequations}

We make the following assumptions of \eqref{ourRMC-noV-rewrite} throughout this section.

\begin{assumption} {($F$ is lower bounded)} \label{assu-lower-bounded}
There exists a finite constant $F^*$, such that
\[
F(X) \geq F^*,\quad \forall X \in \mathcal{M}.
\]
\end{assumption}
Note that for \eqref{ourRMC-noV}, it is easy to see that $F^*=0$.

\begin{assumption} {(Lipschitz Continuity of $\nabla_S f(U,S)$)} \label{assu-Lip-nabla-S}
The gradient $\nabla_S f(U,S)$ is Lipschitz continuous with Lipschitz constant $L_S$. That is
\[
\ba{lll}
\norm{\nabla_S f(U, S_1) - \nabla_S f(U, S_2)}_F \leq  L_S\norm{S_1 - S_2}_F, \\
   \forall S_1, S_2 \in \mathbb{R}^{m \times n}, U \in \mathcal{M}.
   \ea
\]
\end{assumption}

The following assumption is about $\grad_U f(U,S)$, which regards the regularity of the pullback function $\hat{f}(\Delta U, S) = f(\Retr_U(\Delta U), S)$, and differs from the standard Lipschitz continuity assumption because of the retraction operator. This assumption was originally suggested in \cite{boumal2018global}.

\begin{assumption}{(Restricted Lipschitz-type gradient for pullbacks)}\label{assu-pullback}
There exists $L_U \ge 0$ such that, for sequence $(U^k,S^k)_{k\geq 0}$ generated by AManPG (Algorithm \ref{alg:Amanpg}), the pullback function $\hat{f}_k(\Delta U) = f(\Retr_{U^k}(\Delta U),S^{k+1})$ satisfies
\[
\ba{lll}
\left|\hat{f}_k(\Delta U) - [f(U^k, S^{k+1}) + \langle \Delta U, \grad_U f(U^k, S^{k+1})\rangle ] \right| \\
\leq \frac{L_U}{2} \norm{\Delta U}_F^2,  \forall \Delta U \in \T_{U^k}\mathcal{M}.
\ea
\]
\end{assumption}

From the Theorem 4.1 in \cite{yang2014optimality}, we can define the stationary point of problem \eqref{ourRMC-noV-rewrite} as follows.

\begin{definition} {(Stationary point).}
A pair of $(U,S)\in \M\times \br^{m\times n}$ is called a stationary point of problem \eqref{ourRMC-noV-rewrite} if it satisfies the first-order necessary conditions:
\be\label{opt-cond}
0 = \grad_Uf(U, S), \quad 0 \in \nabla_Sf(U, S) + \partial h(S).
\ee
\end{definition}

According to Theorem 4.1 in \cite{yang2014optimality}, the optimality conditions of the subproblems \eqref{AManPG-general-dU} and \eqref{AManPG-general-dS} are
\be\label{opt-cond-sub}
\ba{lll}
0 = \grad_Uf(U^k, S^{k+1}) + \frac{1}{t_U}\Delta U^k,\\
0 \in \frac{1}{t_S}\Delta S^k + \nabla_Sf(U^k, S^k) + \partial h(S^k+\Delta S^k).
\ea
\ee
If $\Delta U^k=0$ and $\Delta S^k=0$, then we know that $S^{k+1}=S^k$, and $(U^k,S^k)$ satisfies \eqref{opt-cond} and thus is a stationary point of \eqref{ourRMC-noV-rewrite}. Therefore, we can use the norm of $(U^k,S^k)$ to measure the closeness to stationary point, and we define the $\epsilon$-stationary point of \eqref{ourRMC-noV-rewrite} as follows.

\begin{definition} {($\epsilon$-stationary point).}
We say that $(U^k,S^k)\in \mathcal{M}\times\br^{m\times n}$ is an $\epsilon$-stationary point of \eqref{ourRMC-noV-rewrite}, if $(\Delta S^k, \Delta U^k)$ given by \eqref{AManPG-general-dS} and \eqref{AManPG-general-dU} satisfies
\be\label{terminate-crit}\|\Delta S^k\|_F^2+\|\Delta U^k\|_F^2\leq \epsilon^2 /L^2,\ee 
where $L := \min\{L_S, L_U\}$. 
\end{definition} 
Now we are ready to analyze the iteration complexity of AManPG for obtaining an $\epsilon$-stationary point of \eqref{ourRMC-noV-rewrite}. First, we prove two lemmas, which show that there is a sufficient reduction of the objective value after each update of $S$ and $U$.

\begin{lemma} \label{lem1}
Assume Assumption \ref{assu-Lip-nabla-S} holds, and the sequence $(S^k,U^k,\Delta S^k, \Delta U^k)$ is generated by AManPG. By choosing $t_S = 1/L_S$ and $\alpha=1$, the following inequality holds
\be\label{lem1-ineq}
F(U^{k}, S^{k+1}) - F(U^{k}, S^k) \le -\frac{L_S}{2} \norm{\Delta S^k}_F^2.
\ee
\end{lemma}

\begin{proof}
Please see Appendix~\ref{proof:lemma1} for details.
\end{proof}

\begin{lemma}\label{lem2}
Assume Assumption \ref{assu-pullback} holds, and the sequence $(S^k,U^k,\Delta S^k, \Delta U^k)$ is generated by AManPG. By choosing $t_U = 1/L_U$ and $\beta=1$, the following inequality holds
\be\label{lem2-ineq}
F(U^{k+1}, S^{k+1}) - F(U^{k}, S^{k+1}) \le -\frac{L_U}{2} \norm{\Delta U^k}_F^2.
\ee
\end{lemma}
\begin{proof}
Please see Appendix~\ref{proof:lemma2} for details.
\end{proof}

Now we are ready to present our main convergence result of AManPG. 
\begin{theorem}\label{thm1}
Assume Assumptions \ref{assu-lower-bounded}, \eqref{assu-Lip-nabla-S} and \eqref{assu-pullback} hold. By choosing $t_U = 1/L_U$, $t_S = 1/L_S$, $\alpha=\beta=1$ in AManPG \eqref{AManPG-general}, every limit point of the sequence $\{U^k,S^k\}$ generated by AManPG \eqref{AManPG-general} is a stationary point of problem \eqref{ourRMC-noV-rewrite}. Moreover, AManPG \eqref{AManPG-general} returns an $\epsilon$-stationary point of problem \eqref{ourRMC-noV-rewrite} in at most $\lceil 2L(F(U^0,S^0) - F^*) / \epsilon^2\rceil$ iterations, where $L := \min(L_S, L_U)$.
\end{theorem}

\begin{proof}
Please see Appendix~\ref{proof:theorem1} for details.
\end{proof}

\section{Numerical Experiments}\label{sec:numerical}

\begin{figure*}[ht]
	\begin{center}
		\minipage{0.25\textwidth}
		\subfigure[Case 1]{\includegraphics[width=0.9\linewidth]{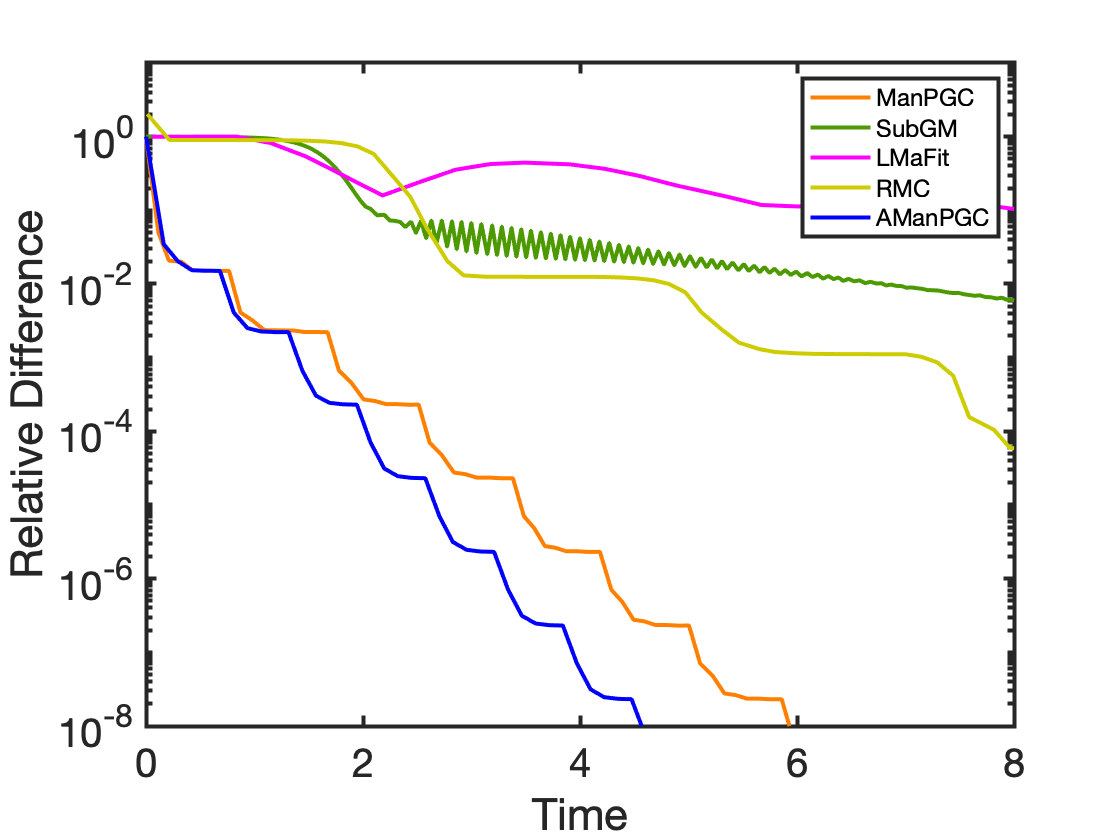}}
		\endminipage\hfill
		\minipage{0.25\textwidth}
		\subfigure[Case 2]{\includegraphics[width=0.9\linewidth]{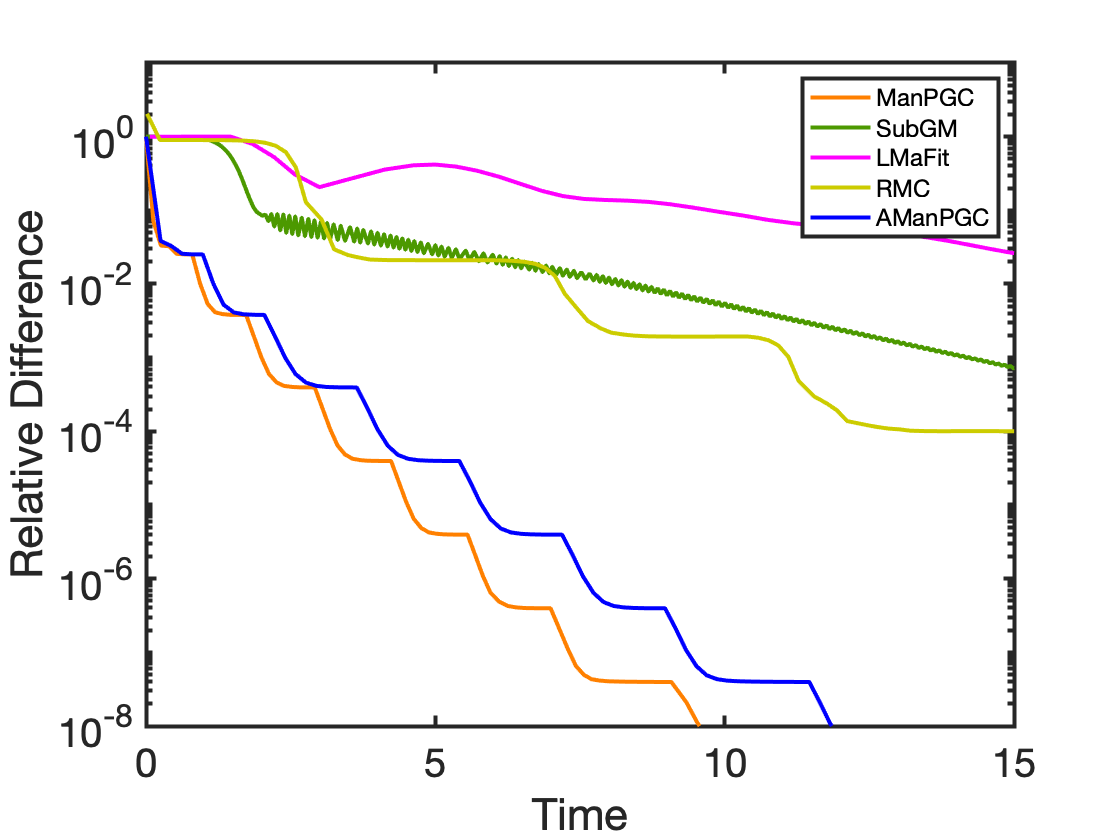}}
		\endminipage\hfill
		\minipage{0.25\textwidth}
		\subfigure[Case 3]{\includegraphics[width=0.9\linewidth]{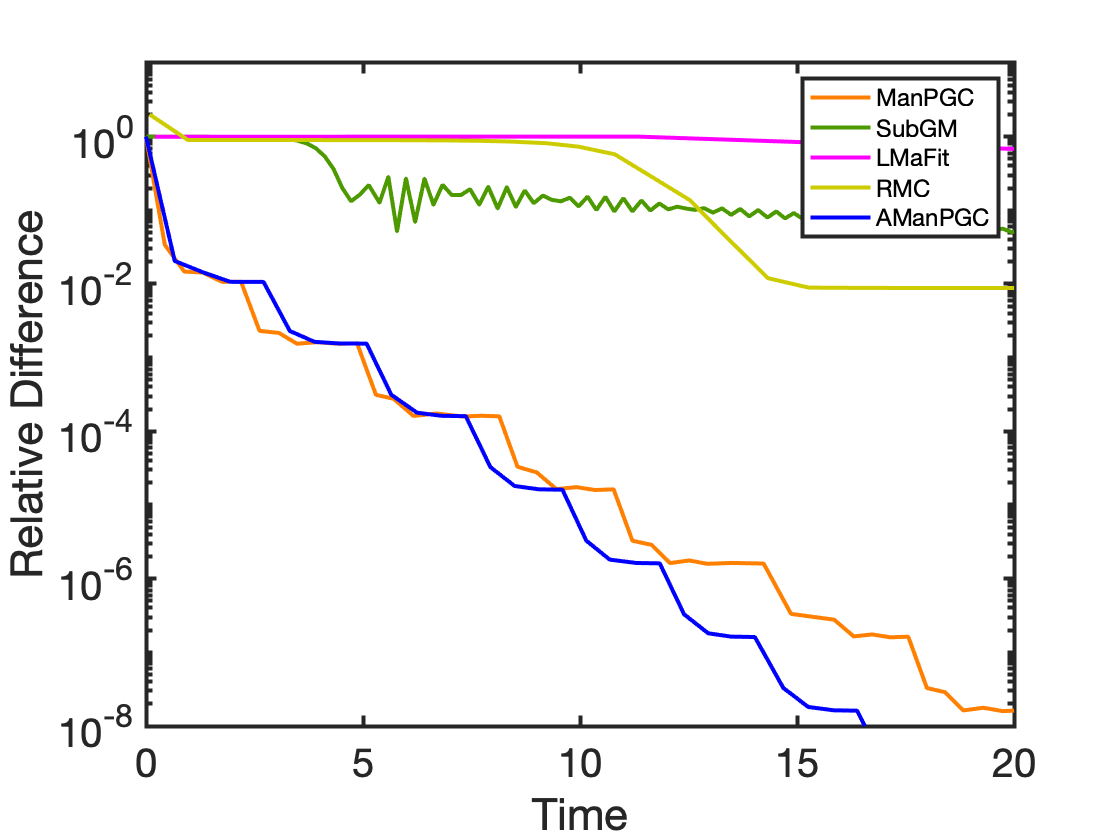}}
		\endminipage\hfill
		\minipage{0.25\textwidth}
		\subfigure[Case 4]{\includegraphics[width=0.9\linewidth]{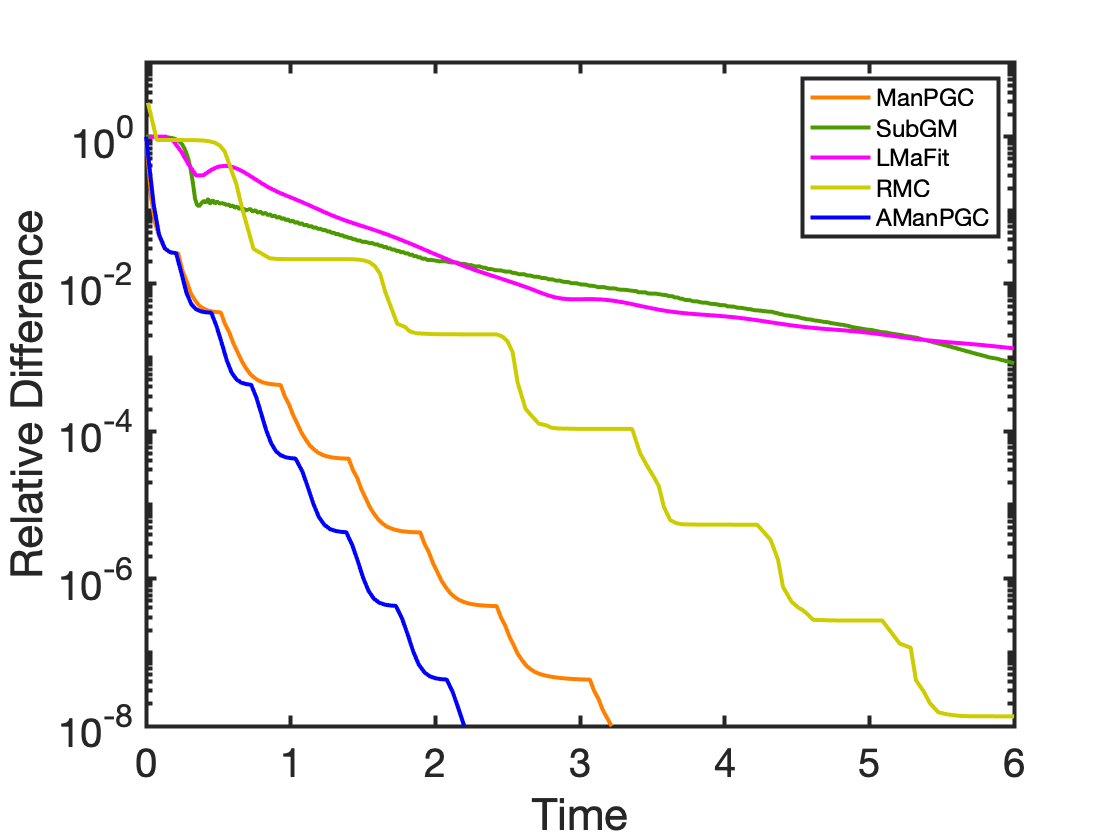}}
		\endminipage\hfill
		\minipage{0.25\textwidth}
		\subfigure[Case 1]{\includegraphics[width=0.9\linewidth]{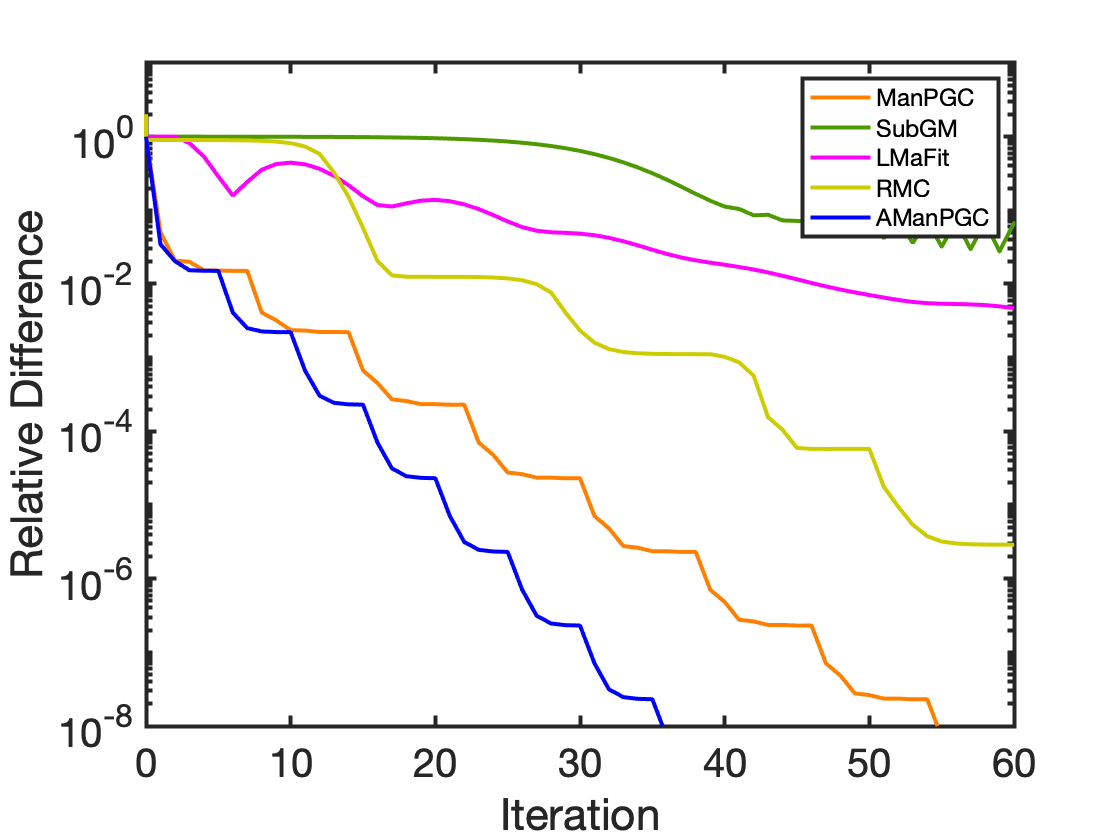}}
		\endminipage\hfill
		\minipage{0.25\textwidth}
		\subfigure[Case 2]{\includegraphics[width=0.9\linewidth]{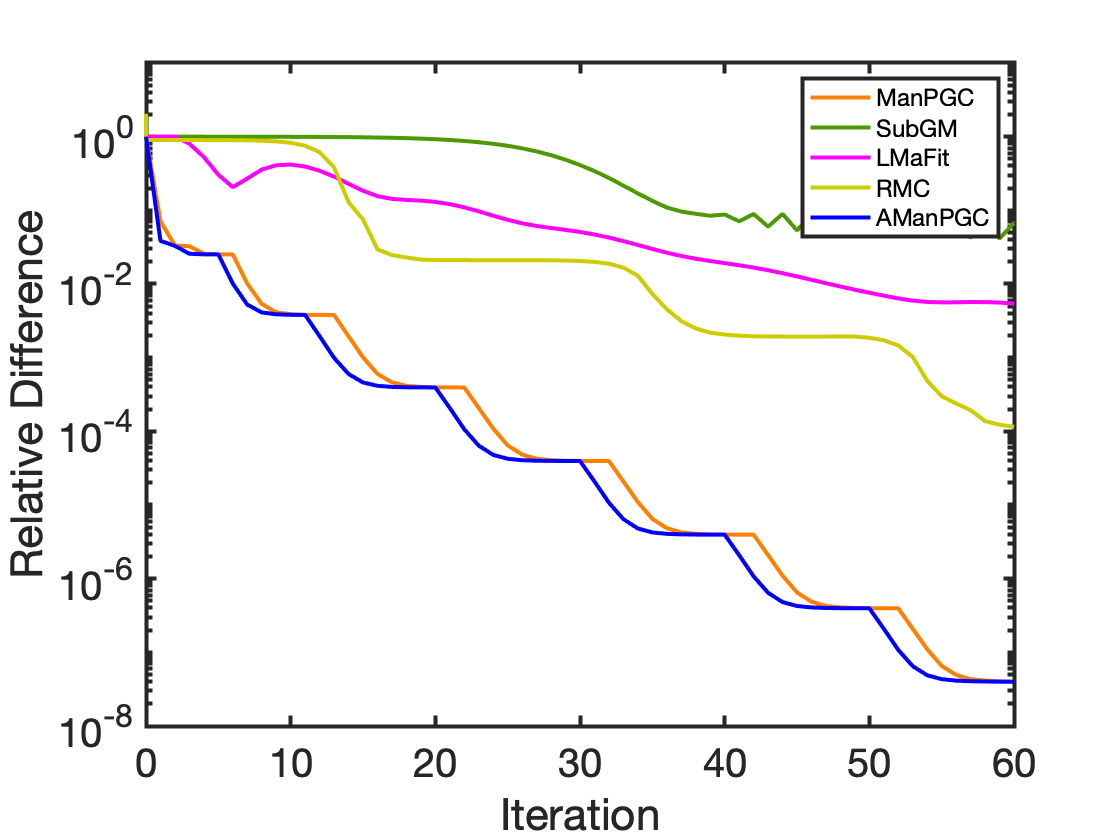}}
		\endminipage\hfill
		\minipage{0.25\textwidth}
		\subfigure[Case 3]{\includegraphics[width=0.9\linewidth]{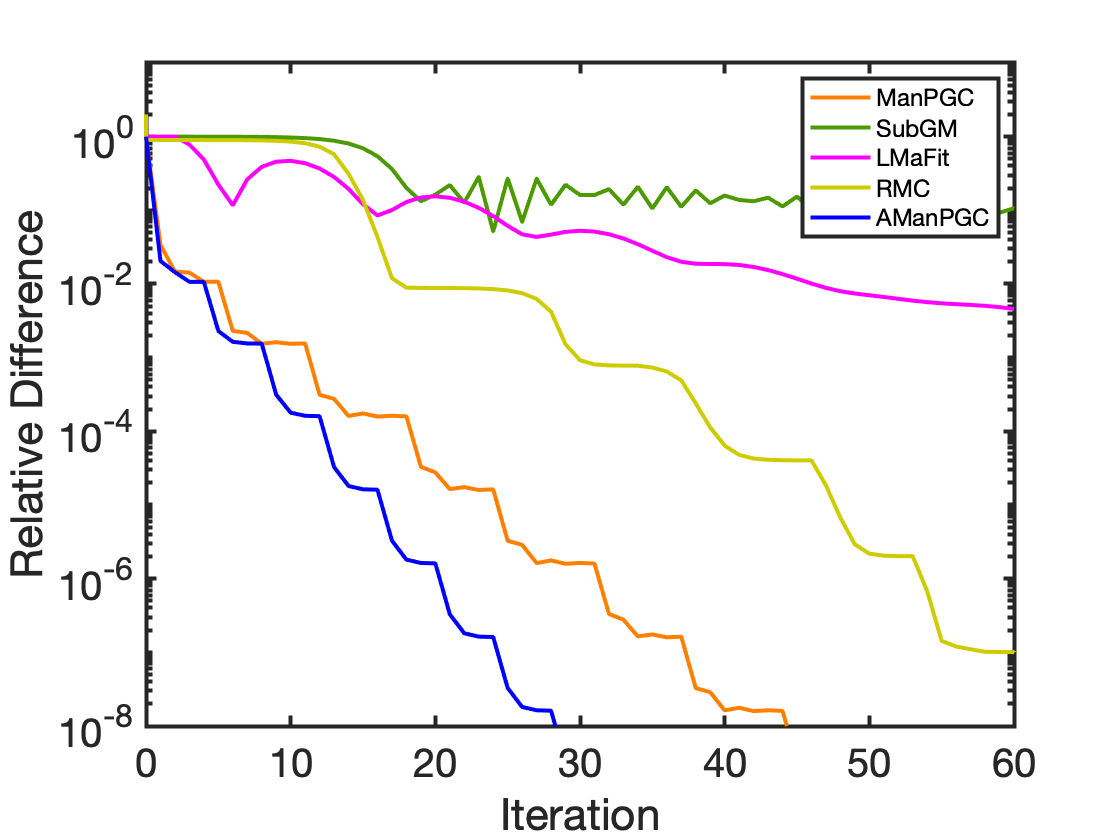}}
		\endminipage\hfill
		\minipage{0.25\textwidth}
		\subfigure[Case 4]{\includegraphics[width=0.9\linewidth]{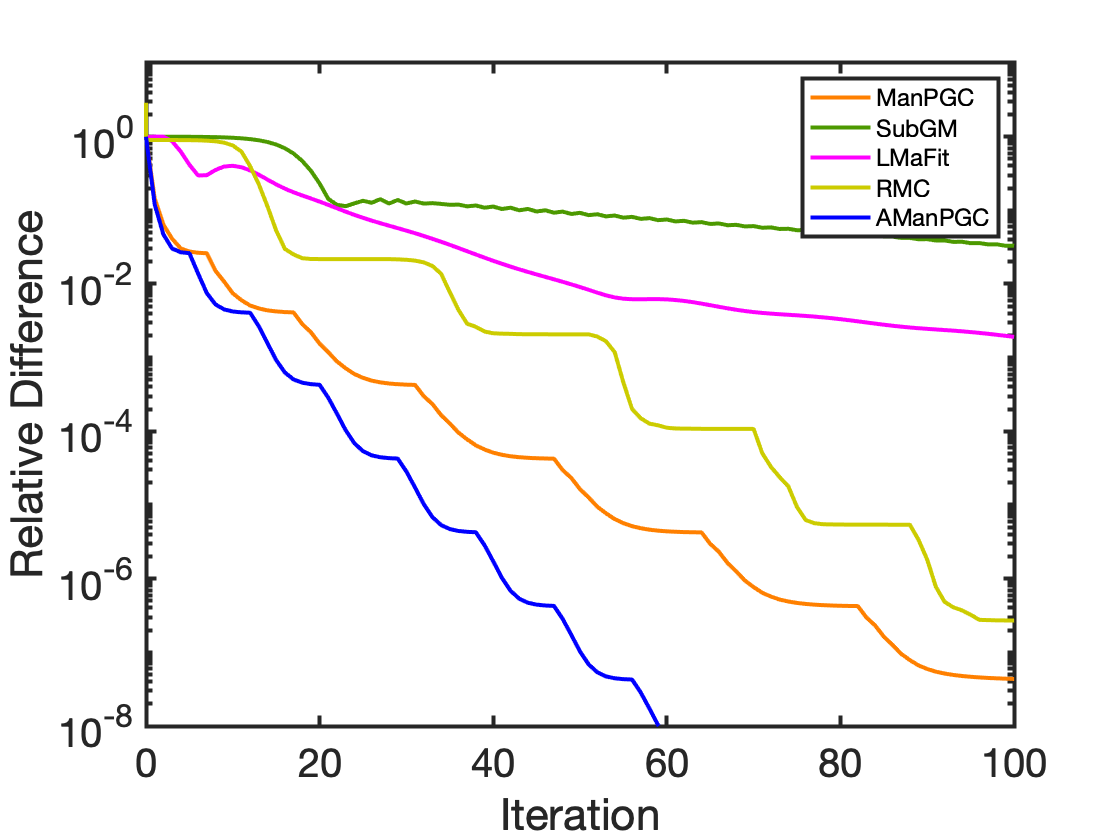}}
		\endminipage\hfill
		\minipage{0.25\textwidth}
		\subfigure[Case 5]{\includegraphics[width=0.9\linewidth]{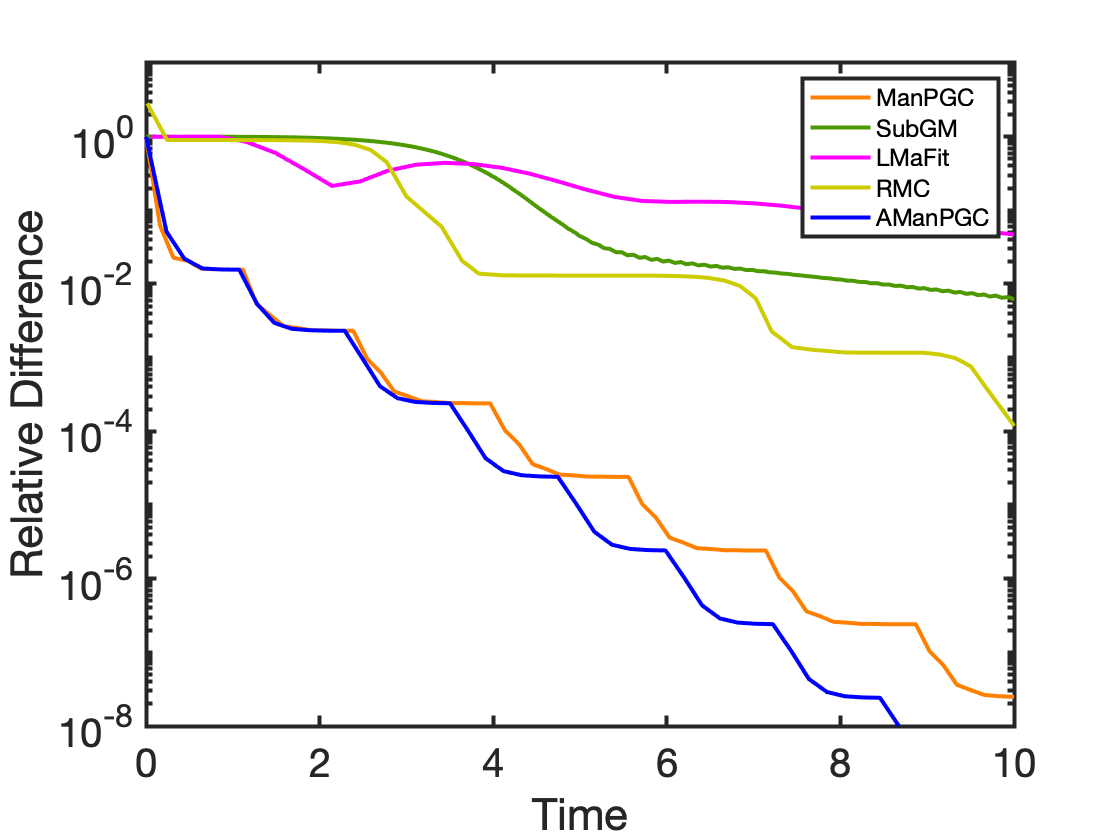}}
		\endminipage\hfill
		\minipage{0.25\textwidth}
		\subfigure[Case 6]{\includegraphics[width=0.9\linewidth]{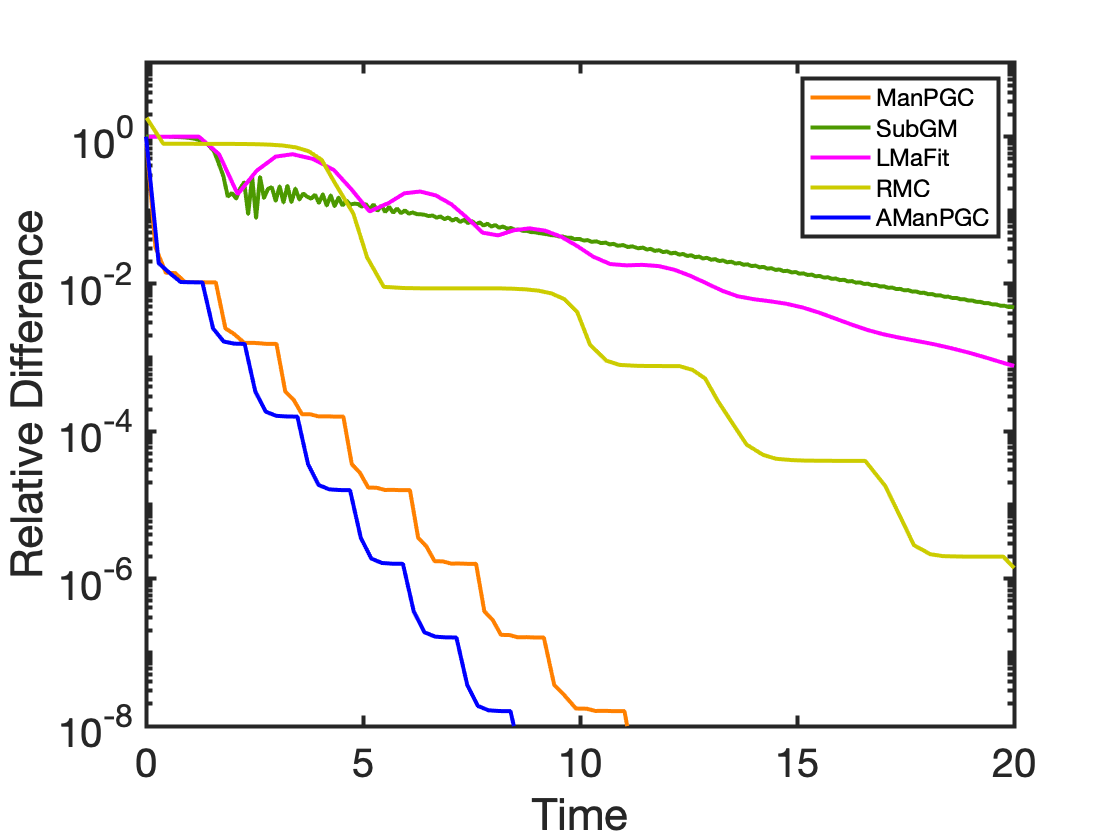}}
		\endminipage\hfill
		\minipage{0.25\textwidth}
		\subfigure[Case 7]{\includegraphics[width=0.9\linewidth]{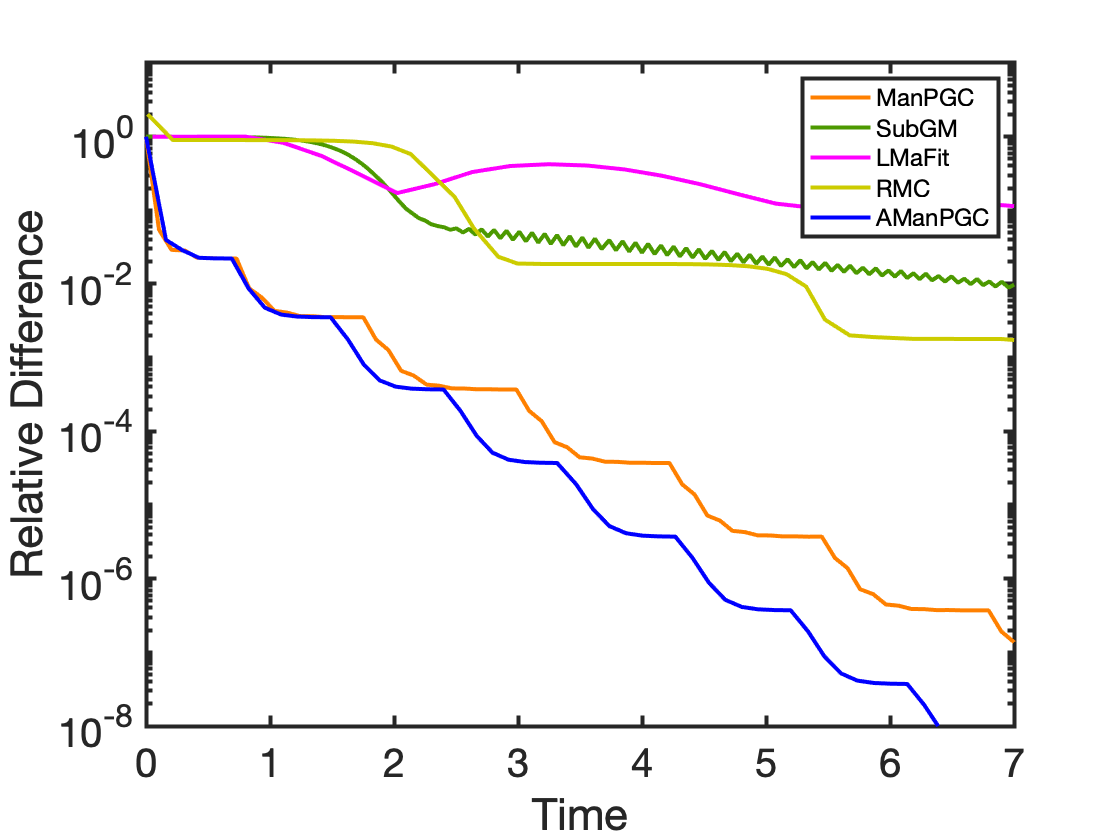}}
		\endminipage\hfill
		\minipage{0.25\textwidth}
		\subfigure[Case 8]{\includegraphics[width=0.9\linewidth]{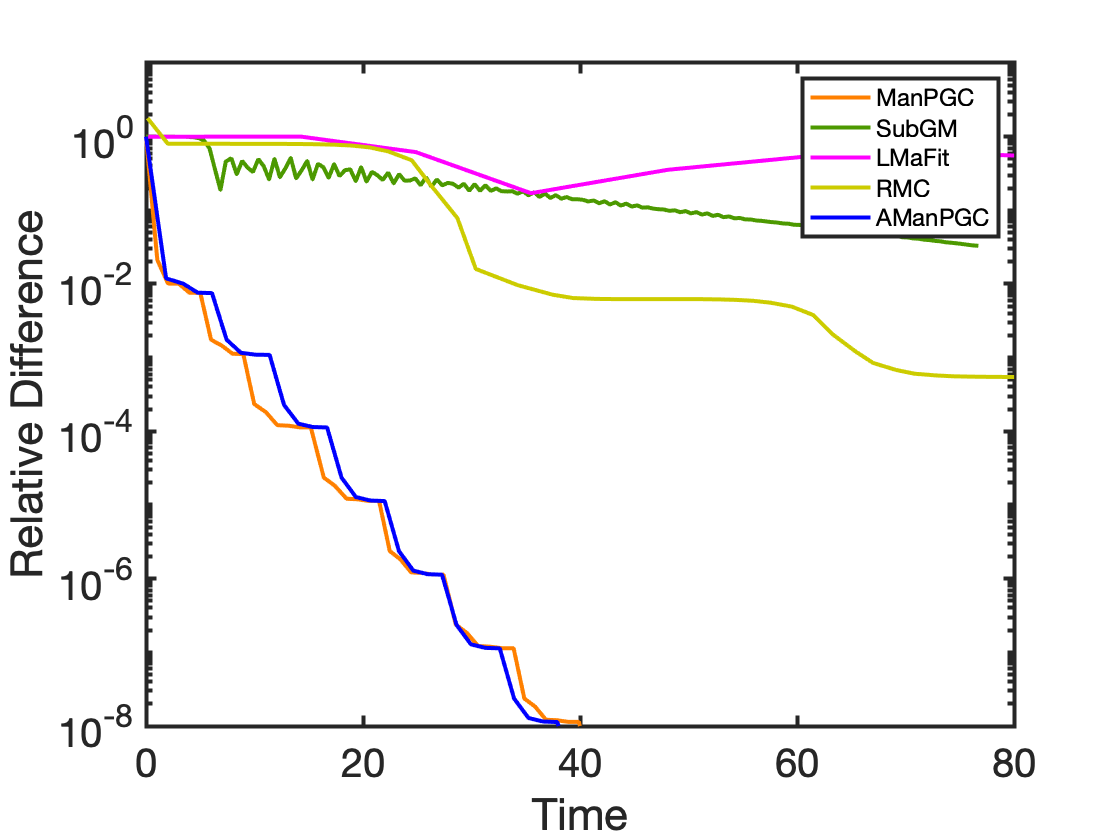}}
		\endminipage\hfill
		\minipage{0.25\textwidth}
		\subfigure[Case 5]{\includegraphics[width=0.9\linewidth]{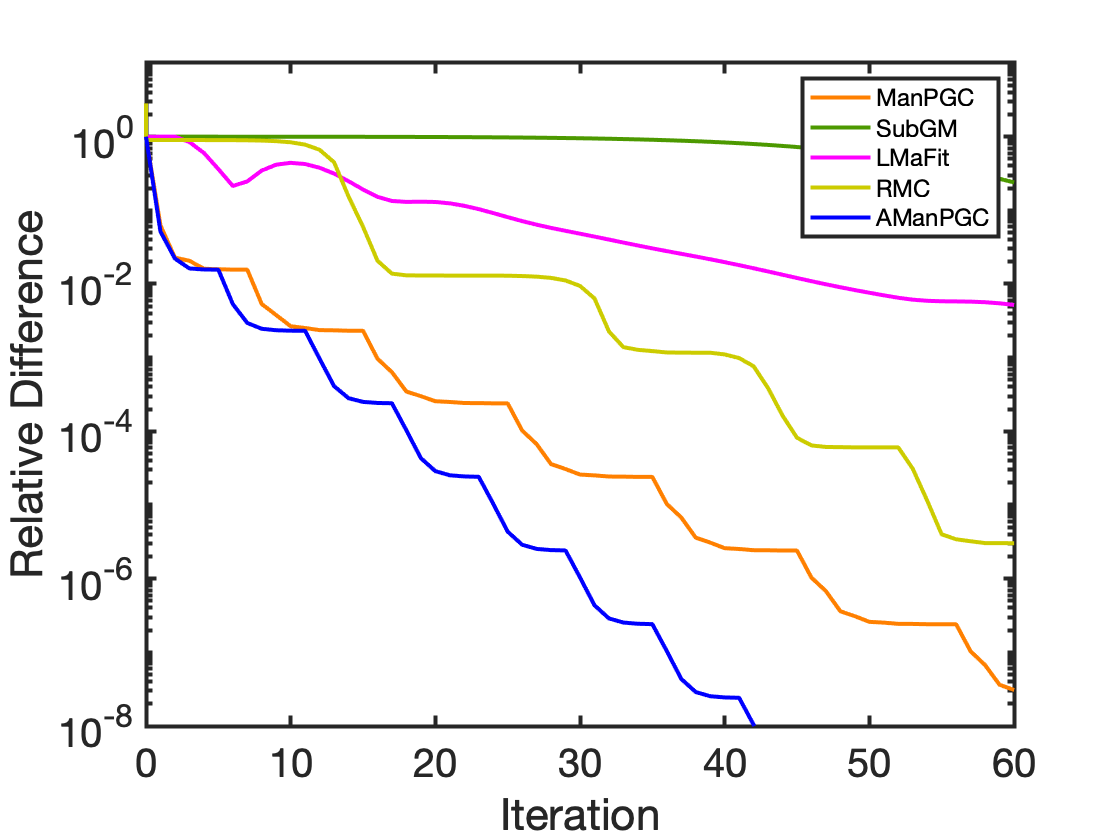}}
		\endminipage\hfill
		\minipage{0.25\textwidth}
		\subfigure[Case 6]{\includegraphics[width=0.9\linewidth]{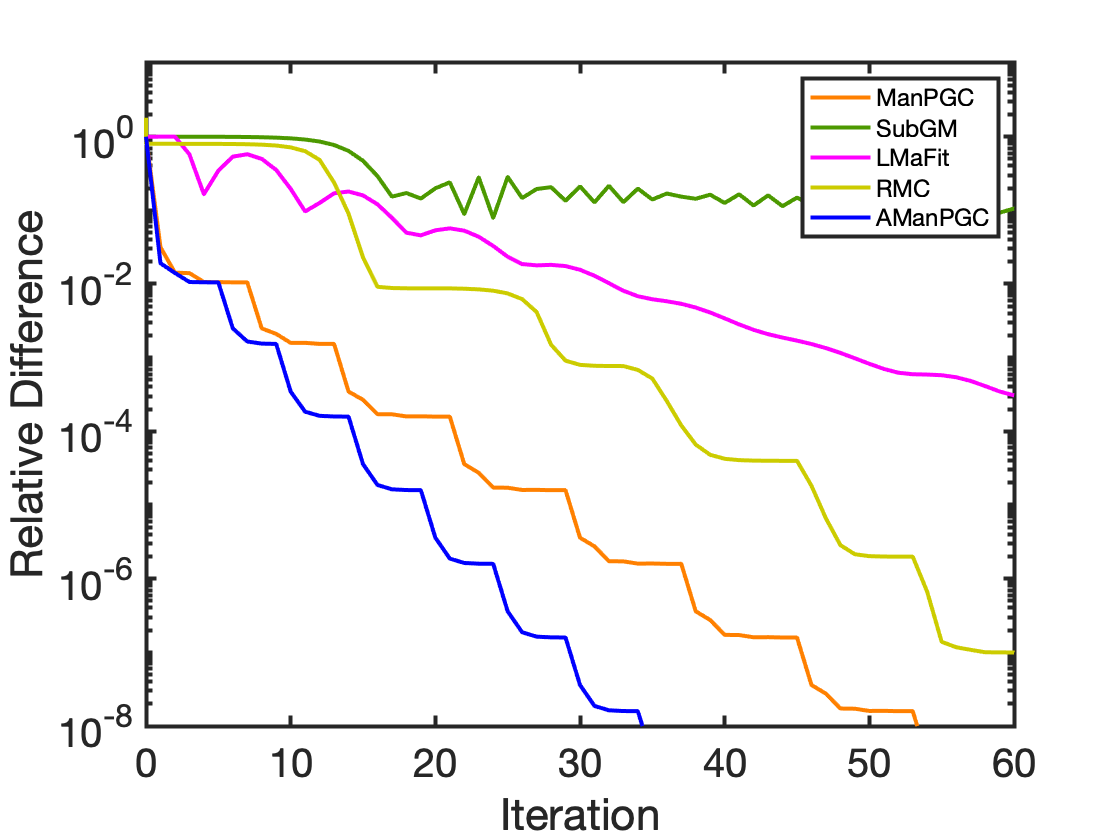}}
		\endminipage\hfill
		\minipage{0.25\textwidth}
		\subfigure[Case 7]{\includegraphics[width=0.9\linewidth]{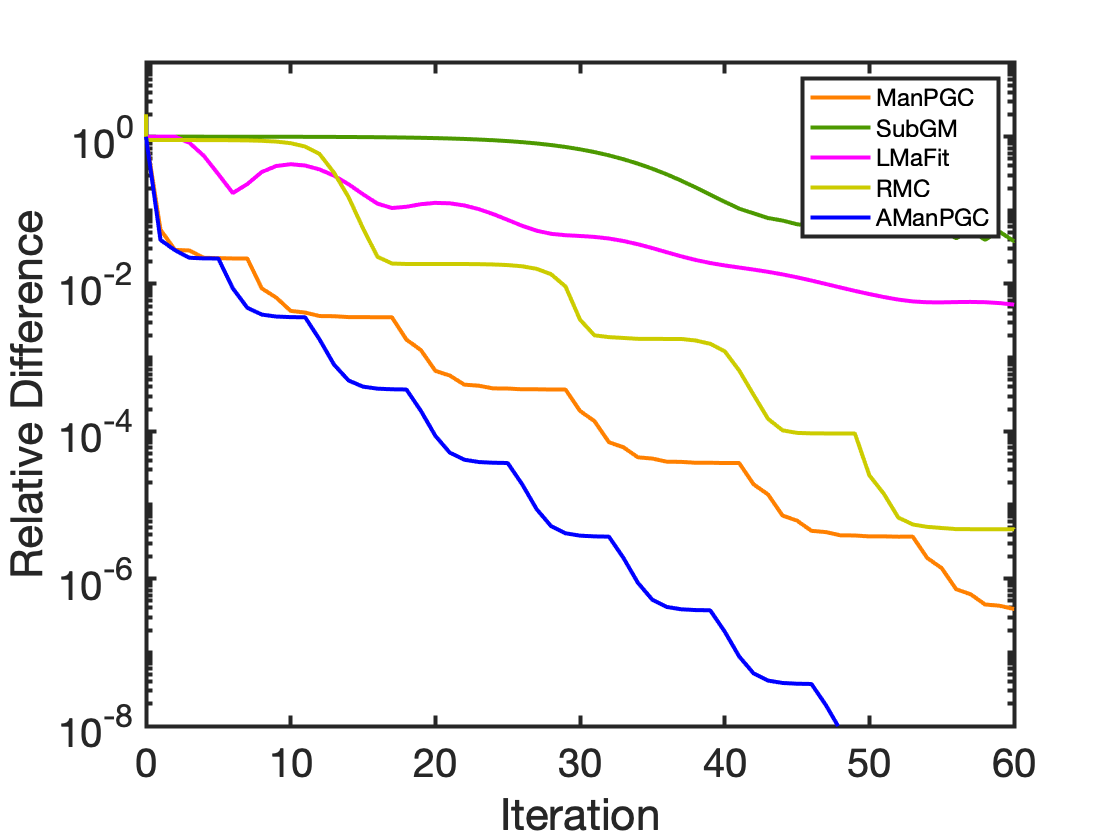}}
		\endminipage\hfill
		\minipage{0.25\textwidth}
		\subfigure[Case 8]{\includegraphics[width=0.9\linewidth]{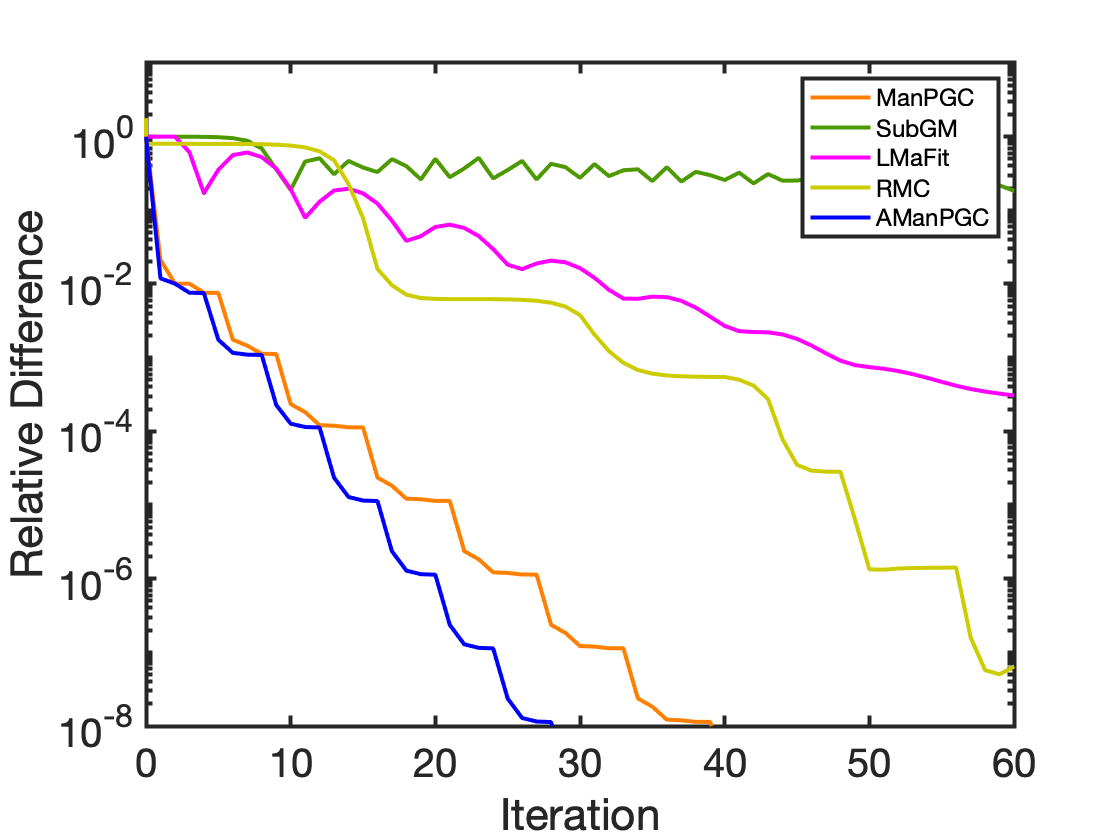}}
		\endminipage\hfill
		\caption{Relative difference for synthetic data on different cases. (a), (b), (c), (d), (i), (j), (k), (l) present the CPU time comparison; (e), (f), (g), (h), (m), (n), (o), (p) present the running iteration comparison.}
		\label{fig:synthetic}
	\end{center}
	\vskip -0.1in
\end{figure*}

In this section, we provide numerical results for both synthetic and real datasets to verify the performance of the proposed algorithms. We focus on comparing our ManPG and AManPG algorithms with some baseline algorithms using the robustness of $\ell_1$-norm, in particular, the subgradient method (SubGM) \cite{li2018nonconvex}, the LMaFit \cite{shen2014augmented} and the Riemannian conjugate gradient method for the smoothed $\ell_1$-norm objective function (RMC) \cite{cambier2016robust}. We use the same continuation framework for ManPG and call it ManPGC. For the SubGM method, we assume that the linear operator $\A$ is a simple projection. For the LMaFit algorithm, we turn off the rank estimation since we assume that the rank is known for all cases. We use the original setting for the RMC algorithm.  All algorithms use same C-Mex code for accelerating the matrix multiplication between a sparse matrix and a full matrix and some other bottleneck computations. Moreover, to guarantee a fair comparison, each algorithm is carefully tuned to achieve its best performance. All experiments were run on Matlab R2018b with a 2.3 GHz Dual-Core Intel Core i5 CPU.

\subsection{Synthetic Data}

We first test our ManPGC and AManPGC algorithms on different cases of synthetic data. After picking the values of $m, n, r$, we generate the ground truth $U^* \in \R^{m \times r}$,  $V^* \in \R^{r \times n}$ with i.i.d. normal entries of zero mean and unit variance. The target matrix is $X^* =U^*V^*$. We then choose a sampling ratio and sample entries uniformly at random to get the observed matrix $M$. Finally, we add a sparse matrix $S^*$, whose nonzero entries are generated by a normal distribution with zero mean and unit variance with a sparsity rate, to the observed matrix $M$.  

\textbf{Parameters:} In our experiment, we observe that by setting $t_S = 1$, we get very good performance. Due to the problem formulation \eqref{ourRMC}, $t_S = 1$ gives us a direct proximal mapping for the $S$ subproblem when we fix $V_{U^k, S^k}$. We tune $t_U$ between $1/\lvert \Omega \rvert$ and $3/\lvert \Omega \rvert$ and set $\gamma_0 = 10, \lambda = 10^{-8}, \mu_1=\mu_2=1/10$ in all experiments. We use different $\epsilon_0$ values specified in the following cases. We may use any random matrix $U_0$ as our initial point, but in practice, we use the singular value decomposition of the observed matrix $M$ as our initial point for all algorithms. 
\begin{figure*}[t]
	\centering
	\minipage{0.16\textwidth}
	\subfigure[Original Image]{\includegraphics[width=0.9\linewidth]{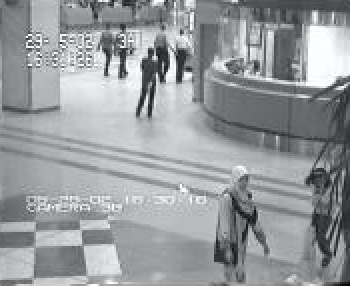}}
	\endminipage\hfill
	\minipage{0.16\textwidth}
	\subfigure[ManPGC]{\includegraphics[width=0.9\linewidth]{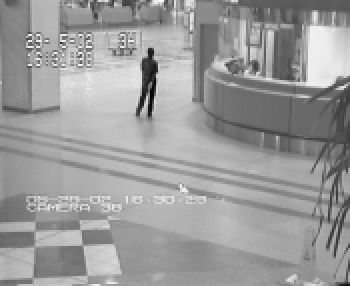}}
	\endminipage\hfill
	\minipage{0.16\textwidth}
	\subfigure[AManPGC]{\includegraphics[width=0.9\linewidth]{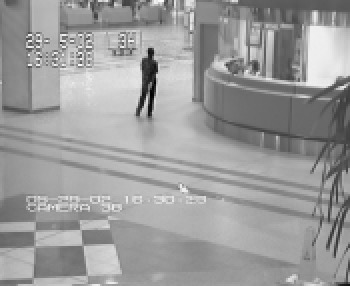}}
	\endminipage\hfill
	\minipage{0.16\textwidth}
	\subfigure[RMC]{\includegraphics[width=0.9\linewidth]{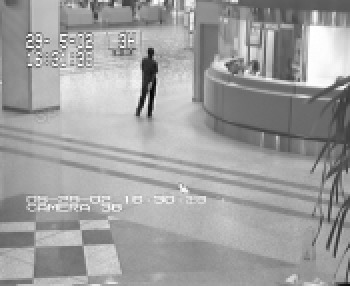}}
	\endminipage\hfill
	\minipage{0.16\textwidth}
	\subfigure[LMaFit]{\includegraphics[width=0.9\linewidth]{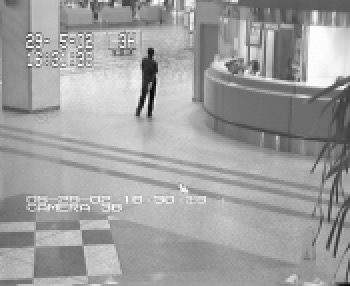}}
	\endminipage\hfill
	\minipage{0.16\textwidth}
	\subfigure[SubGM]{\includegraphics[width=0.9\linewidth]{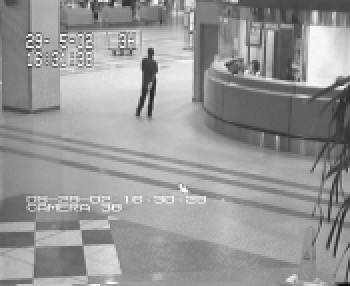}}
	\endminipage\hfill\\
	\minipage{0.16\textwidth}
	\subfigure[Original Image]{\includegraphics[width=0.9\linewidth]{hall_original.png}}
	\endminipage\hfill
	\minipage{0.16\textwidth}
	\subfigure[ManPGC]{\includegraphics[width=0.9\linewidth]{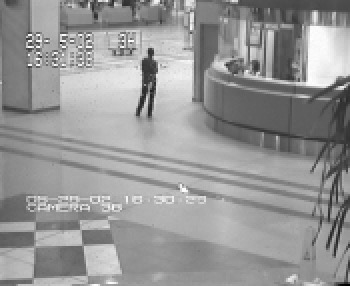}}
	\endminipage\hfill
	\minipage{0.16\textwidth}
	\subfigure[AManPGC]{\includegraphics[width=0.9\linewidth]{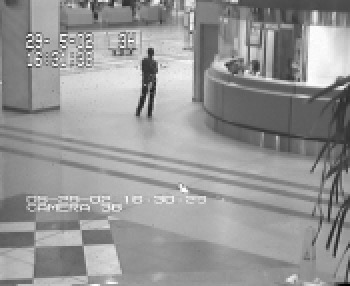}}
	\endminipage\hfill
	\minipage{0.16\textwidth}
	\subfigure[RMC]{\includegraphics[width=0.9\linewidth]{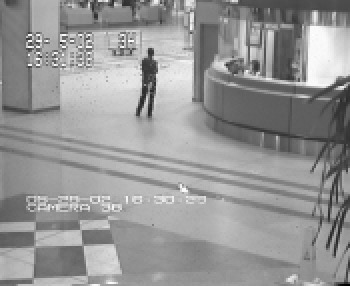}}
	\endminipage\hfill
	\minipage{0.16\textwidth}
	\subfigure[LMaFit]{\includegraphics[width=0.9\linewidth]{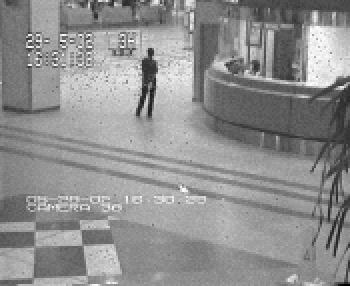}}
	\endminipage\hfill
	\minipage{0.16\textwidth}
	\subfigure[SubGM]{\includegraphics[width=0.9\linewidth]{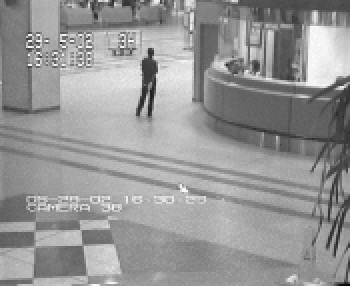}}
	\endminipage\hfill
	\caption{Background estimation for ``Hall of a business building'' video data. The first row are recovered from $50\%$ observed pixels and the second row are recovered from $10\%$ observed pixels . (a), (g) One of the original image frame. (b), (h) Background frame estimated by ManPGC. (c), (i) Background frame estimated by AManPGC. (d), (j) Background frame estimated by RMC \cite{cambier2016robust}. (e), (k) Background frame estimated by LMaFit \cite{shen2014augmented}. (f), (l) Background frame estimated by SubGM \cite{li2018nonconvex}.}
	\label{fig:real_hall}
	\vskip -0.1in
\end{figure*}

\begin{table*}[h] \renewcommand{\arraystretch}{1.3} \caption{CPU time and iteration number comparison (Dataset 1).} \label{table:hall}
\centering
\begin{tabular}{lccccr}
\hline
Algorithm 			& AManPGC 	& ManPGC 	& RMC 	& LMaFit 	& SubGM\\
\hline
CPU time (50\%) 		&1.53		&1.67	&3.55	&15.90	&5.78\\
\hline
Iteration number (50\%)	&9			&10		&14		&31		&50\\
\hline
CPU time (10\%) 		&1.02		&0.91	&1.56	&13.44	&1.07\\
\hline
Iteration number (10\%) 	&11			&12		&30		&27		&60\\
\hline
\end{tabular}
\end{table*}

\begin{figure*}[h]
	\begin{center}
		\minipage{0.16\textwidth}
		\subfigure[Original Image]{\includegraphics[width=0.9\linewidth]{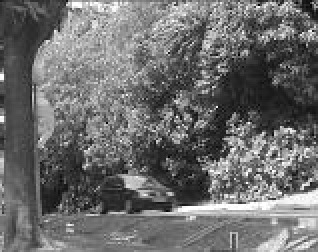}}
		\endminipage\hfill
		\minipage{0.16\textwidth}
		\subfigure[ManPGC]{\includegraphics[width=0.9\linewidth]{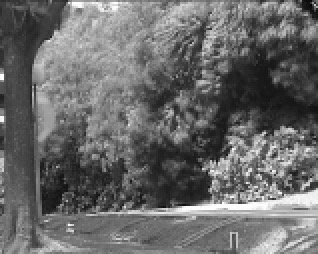}}
		\endminipage\hfill
		\minipage{0.16\textwidth}
		\subfigure[AManPGC]{\includegraphics[width=0.9\linewidth]{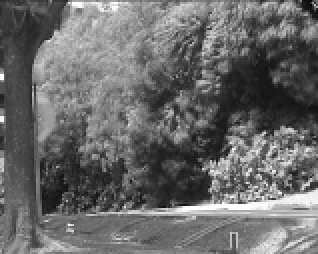}}
		\endminipage\hfill
		\minipage{0.16\textwidth}
		\subfigure[RMC]{\includegraphics[width=0.9\linewidth]{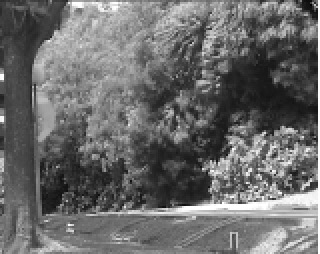}}
		\endminipage\hfill
		\minipage{0.16\textwidth}
		\subfigure[LMaFit]{\includegraphics[width=0.9\linewidth]{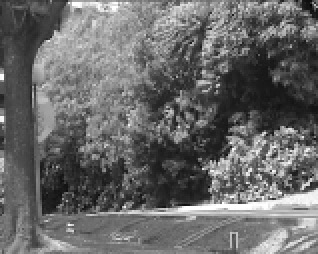}}
		\endminipage\hfill
		\minipage{0.16\textwidth}
		\subfigure[SubGM]{\includegraphics[width=0.9\linewidth]{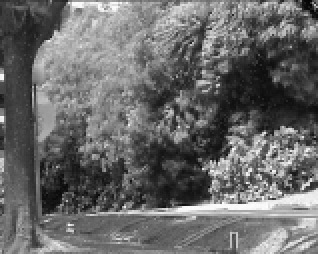}}
		\endminipage\hfill
		\minipage{0.16\textwidth}
		\subfigure[Original Image]{\includegraphics[width=0.9\linewidth]{campus_original.png}}
		\endminipage\hfill
		\minipage{0.16\textwidth}
		\subfigure[ManPGC]{\includegraphics[width=0.9\linewidth]{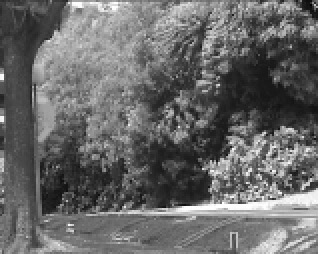}}
		\endminipage\hfill
		\minipage{0.16\textwidth}
		\subfigure[AManPGC]{\includegraphics[width=0.9\linewidth]{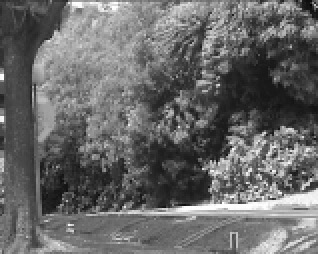}}
		\endminipage\hfill
		\minipage{0.16\textwidth}
		\subfigure[RMC]{\includegraphics[width=0.9\linewidth]{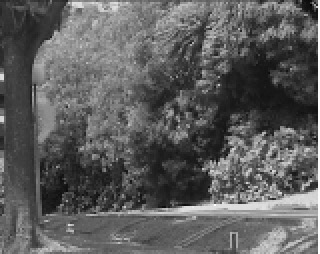}}
		\endminipage\hfill
		\minipage{0.16\textwidth}
		\subfigure[LMaFit]{\includegraphics[width=0.9\linewidth]{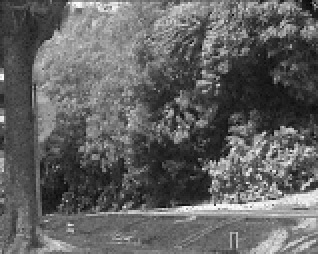}}
		\endminipage\hfill
		\minipage{0.16\textwidth}
		\subfigure[SubGM]{\includegraphics[width=0.9\linewidth]{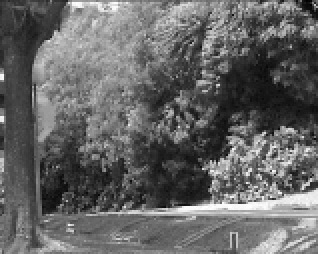}}
		\endminipage\hfill
		\caption{Background estimation for ``Campus Trees'' video data. The first row are recovered from $50\%$ observed pixels and the second row are recovered from $10\%$ observed pixels . (a), (g) One of the original image frame. (b), (h) Background frame estimated by ManPGC. (c), (i) Background frame estimated by AManPGC. (d), (j) Background frame estimated by RMC \cite{cambier2016robust}. (e), (k) Background frame estimated by LMaFit \cite{shen2014augmented}. (f), (l) Background frame estimated by SubGM \cite{li2018nonconvex}.}
		\label{fig:real_campus}
	\end{center}
	\vskip -0.1in
\end{figure*}

\begin{table*}[h] \renewcommand{\arraystretch}{1.3} \caption{CPU time and iteration number comparison (Dataset 2).} \label{table:campus}
\centering
\begin{tabular}{lccccr}
\hline
Algorithm 			& AManPGC 	& ManPGC 	& RMC 	& LMaFit 	& SubGM\\
\hline
CPU time (50\%) 		&1.76		&1.56	&3.21	&13.42	&3.78\\
\hline
Iteration number (50\%)	&7			&7		&17		&26		&50\\
\hline
CPU time (10\%) 		&0.82		&0.67	&2.31	&10.93	&1.76\\
\hline
Iteration number (10\%) 	&9			&11		&26		&28		&50\\
\hline
\end{tabular}
\end{table*}

We then test our algorithms in the following settings:\\
\textbf{Case 1:} We pick $m = n = 5000$, $r = 5$, sampling ratio of around $10\%$ and sparsity of $10\%$. We tune $t_U = 2/\lvert \Omega \rvert, \epsilon_0 = 30$ for both AManPGC and ManPGC.\\
\textbf{Case 2:} We pick $m =1000,  n = 30000$, $r = 5$. sampling ratio of around $10\%$ and sparsity of $10\%$. We tune $t_U = 2/\lvert \Omega \rvert, \epsilon_0 = 30$ for both AManPGC and ManPGC.\\ 
\textbf{Case 3:} We pick $m = n = 10000$, $r = 5$. sampling ratio of around $10\%$ and sparsity of $10\%$. We tune $t_U = 2/\lvert \Omega \rvert, \epsilon_0 = 30$ for both AManPGC and ManPGC.\\ 
\textbf{Case 4:} We pick $m = n = 2000$, $r = 10$, sampling ratio of around $20\%$ and sparsity of $10\%$. We tune $t_U = 2/\lvert \Omega \rvert, \epsilon_0 = 30$ for AManPGC and $t_U = 1.6/\lvert \Omega \rvert, \epsilon_0 = 20$ for ManPGC.\\
\textbf{Case 5:} We pick $m = n = 5000$, $r = 10$, sampling ratio of around $10\%$ and sparsity of $10\%$. We tune $t_U = 2/\lvert \Omega \rvert, \epsilon_0 = 30$ for both AManPGC and ManPGC.\\
\textbf{Case 6:} We pick $m = n = 5000$, $r = 5$, sampling ratio of around $20\%$ and sparsity of $10\%$. We tune $t_U = 2/\lvert \Omega \rvert, \epsilon_0 = 30$ for both AManPGC and ManPGC.\\
\textbf{Case 7:} We pick $m = n = 5000$, $r = 5$, sampling ratio of around $10\%$ and sparsity of $20\%$. We tune $t_U = 2/\lvert \Omega \rvert, \epsilon_0 = 30$ for both AManPGC and ManPGC.\\
\textbf{Case 8:} We pick $m = n = 10000$, $r = 5$, sampling ratio of around $20\%$ and sparsity of $10\%$.We tune $t_U = 2/\lvert \Omega \rvert, \epsilon_0 = 100$ for both AManPGC and ManPGC. \\

\begin{figure*}[h]
	\centering
		\minipage{0.16\textwidth}
		\subfigure[Original Image]{\includegraphics[width=0.9\linewidth]{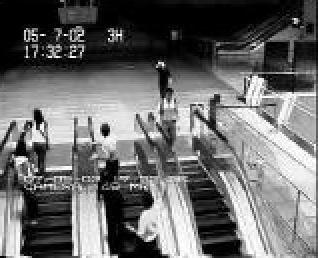}}
		\endminipage\hfill
		\minipage{0.16\textwidth}
		\subfigure[ManPGC]{\includegraphics[width=0.9\linewidth]{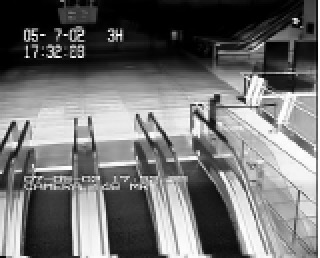}}
		\endminipage\hfill
		\minipage{0.16\textwidth}
		\subfigure[AManPGC]{\includegraphics[width=0.9\linewidth]{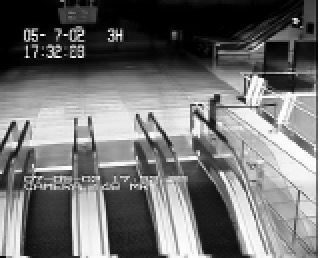}}
		\endminipage\hfill
		\minipage{0.16\textwidth}
		\subfigure[RMC]{\includegraphics[width=0.9\linewidth]{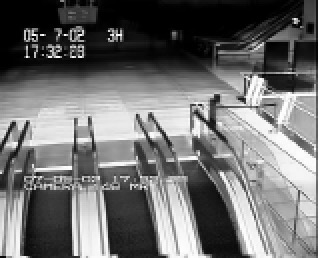}}
		\endminipage\hfill
		\minipage{0.16\textwidth}
		\subfigure[LMaFit]{\includegraphics[width=0.9\linewidth]{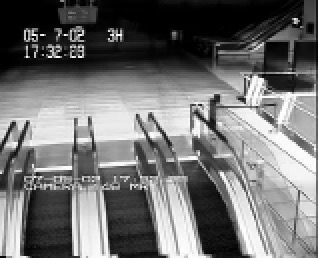}}
		\endminipage\hfill
		\minipage{0.16\textwidth}
		\subfigure[SubGM]{\includegraphics[width=0.9\linewidth]{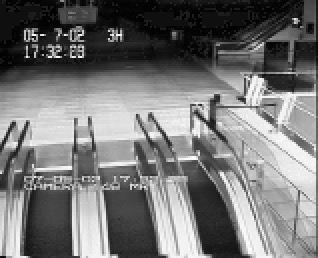}}
		\endminipage\hfill
		\minipage{0.16\textwidth}
		\subfigure[Original Image]{\includegraphics[width=0.9\linewidth]{airport_original.png}}
		\endminipage\hfill
		\minipage{0.16\textwidth}
		\subfigure[ManPGC]{\includegraphics[width=0.9\linewidth]{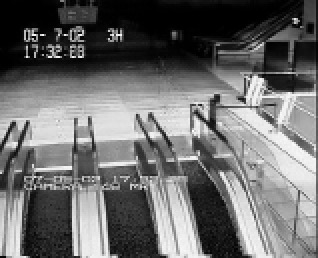}}
		\endminipage\hfill
		\minipage{0.16\textwidth}
		\subfigure[AManPGC]{\includegraphics[width=0.9\linewidth]{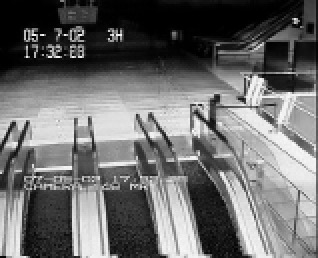}}
		\endminipage\hfill
		\minipage{0.16\textwidth}
		\subfigure[RMC]{\includegraphics[width=0.9\linewidth]{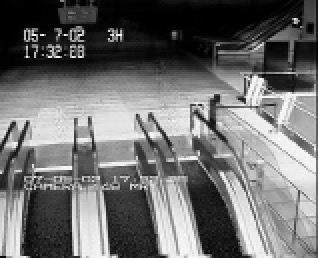}}
		\endminipage\hfill
		\minipage{0.16\textwidth}
		\subfigure[LMaFit]{\includegraphics[width=0.9\linewidth]{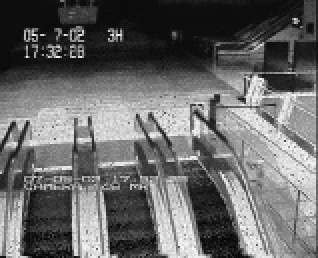}}
		\endminipage\hfill
		\minipage{0.16\textwidth}
		\subfigure[SubGM]{\includegraphics[width=0.9\linewidth]{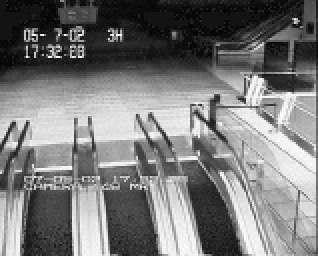}}
		\endminipage\hfill
		\caption{Background estimation for ``Airport Elevator'' video data. The first row are recovered from $50\%$ observed pixels and the second row are recovered from $10\%$ observed pixels . (a), (g) One of the original image frame. (b), (h) Background frame estimated by ManPGC. (c), (i) Background frame estimated by AManPGC. (d), (j) Background frame estimated by RMC \cite{cambier2016robust}. (e), (k) Background frame estimated by LMaFit \cite{shen2014augmented}. (f), (l) Background frame estimated by SubGM \cite{li2018nonconvex}.}
		\label{fig:real_airport}
	\vskip -0.1in
\end{figure*}

\begin{table*}[h] \renewcommand{\arraystretch}{1.3} \caption{CPU time and iteration number comparison (Dataset 3).} \label{table:airport}
\centering
\begin{tabular}{lccccr}
\hline
Algorithm 			& AManPGC 	& ManPGC 	& RMC 	& LMaFit 	& SubGM\\
\hline
CPU time (50\%) 		&1.61		&1.37	&2.70	&9.97	&3.65\\
\hline
Iteration number (50\%)	&8			&8		&26		&25		&60\\
\hline
CPU time (10\%) 		&0.74		&0.95	&1.72	&9.12	&1.08\\
\hline
Iteration number (10\%) 	&15			&15		&28		&26		&60\\
\hline
\end{tabular}
\end{table*}

In the $k$-th iteration, we calculate the relative difference for each method as 
\begin{equation}
\begin{aligned}
\text{Relative Difference}(k) = \frac{\norm{U^kV^k - X^*}_F}{\norm{X^*}_F},
\end{aligned}
\end{equation}
 and report the Relative Difference versus both CPU time (second) and iteration number in Figure \ref{fig:synthetic}. 

From Figure \ref{fig:synthetic}, we can see that the proposed ManPGC and AManPGC algorithms outperform all other baseline algorithms on both CPU time and number of iterations for all cases. It also shows that for most cases, AManPGC performs better than ManPGC.

\subsection{Real Data: Video Background Estimation from Partial Observation}
We now evaluate the performance of ManPGC and AManPGC for video background estimation~\cite{li2004statistical}. By stacking the columns of each frame into a long vector, we obtain a low-rank plus sparse matrix with the almost fixed background as the low-rank part. In this case, we observe that even if we only have partial observation for each video frame, we still recover the whole background with very good quality. Additionally, using partial observation speeds up the background reconstruction process since we need less computation in each iteration for all algorithms. We apply all algorithms for the background estimation to three surveillance video datasets: ``Hall of a business building'',  ``Campus Trees'' and ``Airport Elevator''. In our implementation, the observed pixels are randomly picked from the stacked matrix. 

\textbf{Dataset 1:} ``Hall of a business building'' video is a sequence of $300$ grayscale frames of size $144 \times 176$. So the matrix $X^*\in \R^{25344\times 300}$. \\
\textbf{Dataset 2:} ``Campus Trees'' video is a sequence of $994$ grayscale frames of size $128 \times 160$. So the matrix $X^*\in \R^{ 20480 \times 994}$. \\
\textbf{Dataset 3:}``Airport Elevator'' video is a sequence of $997$ grayscale frames of size $130 \times 160$. So the matrix $X^*\in \R^{ 20800 \times 997}$. \\

We set $r = 2$ and test two cases when we have $50\%$ and $10\%$ observed pixels for each dataset. We terminate each algorithm when the recovered matrix is stable. Specifically, we stop each algorithm when the following inequality is satisfied:
\begin{equation}
\begin{aligned}
\frac{\norm{U^kV^k - U^{k-1}V^{k-1}}_F}{\norm{U^{k-1}V^{k-1}}_F} \le \delta,
\end{aligned}
\end{equation}
where we choose $ \delta = 0.01$ for all cases.
We report the background pictures, the running time and the number of iterations in Figure \ref{fig:real_hall}, Figure \ref{fig:real_campus}, Figure \ref{fig:real_airport}, and Table \ref{table:hall}, Table \ref{table:campus}, Table \ref{table:airport}.

From Figure \ref{fig:real_hall}, Figure \ref{fig:real_campus}, Figure \ref{fig:real_airport}, and Table \ref{table:hall}, Table \ref{table:campus}, Table \ref{table:airport}, we conclude that when we have $50\%$ observed pixels, all algorithms can recover the background to a very high quality. Furthermore, the proposed ManPGC and AManPGC algorithms can recover the background using the least number of iterations and run the fastest among all algorithms. We see that $10\%$ observed pixels also can give a good recovery and it further reduces the running time and the iteration numbers for all algorithms. In each case, our proposed ManPGC and AManPGC algorithms still run the fastest. It is reported in  \cite{zhang2018unified} that when we have full observation of the RPCA, recovering the background for Dataset 1 takes at least 18 seconds. Here by only using partial observations, we finish the same task in less than one second, which shows the advantages of partial observation background recovery.

\section{Conclusion} \label{sec:conclution}
In this paper, we have proposed a new formulation for RMC over Grassmann Manifold. Inspired by recent work of ManPG, we have developed a new algorithm called AManPGC for solving this nonconvex nonsmooth manifold optimization problem. We have provided rigorous analysis for the convergence of AManPG algorithm. In our numerical experiments, we have tested our proposed formulation and algorithms for both synthetic and real datasets. Both experiments show that our proposed methods outperform the state-of-the-art methods.

\appendices
\section{Proof of Lemma \ref{lem1}} \label{proof:lemma1}

\begin{proof} For fixed $U\in\M$ and $S\in\br^{m\times n}$, define
\[
g_{U,S}(T) := \langle \nabla_Sf(U, S), T \rangle + \frac{1}{2t_S} \norm{T}_F^2 + h(S + T).
\]
It is obvious that $g_{U,S}$ is $(1/t_S)$-strongly convex, so we have
\be\label{strong-conv}\ba{lll}
g_{U,S}(T_1) \geq& g_{U,S}(T_2) +  \langle \partial g_{U,S}(T_2), T_1-T_2 \rangle +\\
 &\frac{1}{2t_S} \norm{T_1-T_2}_F^2, \quad \forall T_1, T_2\in \mathbb{R}^{m \times n}.
\ea\ee
By letting $T_1=0$, $T_2=\Delta S^k$ in \eqref{strong-conv}, we have
\be\label{ineq3}\ba{lll}
g_{U^k,S^k}(0) &\geq g(\Delta S^k) - \langle \partial g_{U^k,S^k}(\Delta S^k), \Delta S^k \rangle \\&\hspace{4mm}+ \frac{1}{2t_S} \norm{\Delta S^k}_F^2\\
& = g_{U^k,S^k}(\Delta S^k) + \frac{1}{2t_S} \norm{\Delta S^k}_F^2,
\ea
\ee
where the equality is from the optimality condition of \eqref{AManPG-general-dS}, i.e., $0 \in \partial g_{U^k,S^k}(\Delta S^k)$. 
Using Assumption \ref{assu-Lip-nabla-S}, we can get
\begin{eqnarray}\label{ineq2}
f(U^k,S^{k+1}) - f(U^k,S^k) \le && \hspace{-5mm}\langle \nabla_S f(U^k, S^k), \Delta S^k \rangle\nonumber \\&&+ \frac{L_S}{2} \norm{\Delta S^k}_F^2.
\end{eqnarray}
Therefore
\[
\begin{aligned}
&F(U^k, S^{k+1}) - F(U^k, S^k)\\  &= f(U^k, S^{k+1}) - f(U^k, S^k) + h(S^k + \Delta S^k) - h(S^k)\\
&\leq \langle \nabla_S f( U^k, S^k), \Delta S^k \rangle + \frac{L_S}{2} \norm{\Delta S^k}_F^2 \\
 &+ h(S^k +\Delta S^k) - h(S^k)\\
& =  \frac{L_S}{2} \norm{\Delta S^k}_F^2 + g(\Delta S^k) - \frac{1}{2t_S} \norm{\Delta S^k}_F^2  - g(0)\\
&\leq \left(\frac{L_S}{2} - \frac{1}{t_S}\right)\norm{\Delta S^k}_F^2\\
& = -\frac{L_S}{2} \norm{\Delta S^k}_F^2,
\end{aligned}
\]
where the first inequality comes from \eqref{ineq2}, and the second inequality is from \eqref{ineq3}. This completes the proof.
\end{proof}

\section{Proof of Lemma \ref{lem2}}\label{proof:lemma2}

\begin{proof}
From Assumption \ref{assu-pullback}, we have
\[
\begin{aligned}
&F(U^{k+1}, S^{k+1}) - F(U^k, S^{k+1}) \\
&= f(U^{k+1}, S^{k+1}) - f(U^k, S^{k+1}) \\
& \leq f(Retr_{U^k}(\Delta U^k), S^{k+1}) - f(U^k, S^{k+1}) \\
& \leq  \langle \Delta U^k, \grad_U f(U^k, S^{k+1}) \rangle + \frac{L_U}{2} \norm{\Delta U^k}_F^2\\
& \leq \left(\frac{L_U}{2} - \frac{1}{t_U}\right)\norm{\Delta U^k}_F^2\\
& = -\frac{L_U}{2}\norm{\Delta U^k}_F^2,
\end{aligned}
\]
where the second inequality is due to \eqref{opt-cond-sub}. This completes the proof.
\end{proof}

\section{Proof of Theorem \ref{thm1} }\label{proof:theorem1}

\begin{proof} Combining \eqref{lem1-ineq} and \eqref{lem2-ineq} yields,
\be\label{thm-long-ineq}
\begin{aligned}
&F(U^{k+1},S^{k+1}) - F(U^k,S^k) \\
& =  F(U^{k+1}, S^{k+1}) -F(U^{k}, S^{k+1}) + F(U^{k}, S^{k+1}) \\&\hspace{8mm}- F(U^k, S^k)\\
&\leq -\frac{L_S}{2}\norm{\Delta S^k}_F^2 -\frac{L_U}{2}\norm{\Delta U^k}_F^2\\
&\leq - \frac{L}{2}\left(\norm{\Delta S^k}_F^2 + \norm{\Delta U^k}_F^2\right).
\end{aligned}
\ee
Since $F$ is decreasing and bounded below, we have
\[
\lim_{k \to \infty} \left(\norm{\Delta S^k}_F^2 + \norm{\Delta U^k}_F^2\right) = 0.
\]
It follows that every limit point of $\{(U^k,S^k)\}$ is a stationary point of \eqref{ourRMC-noV-rewrite}. Moreover, if AManPG \eqref{AManPG-general} does not terminate after $K$ iterations, i.e., \eqref{terminate-crit} is not satisfied, we have
\[\left(\norm{\Delta S^k}_F^2 + \norm{\Delta U^k}_F^2\right) > \epsilon^2/L^2, \mbox{ for } k =0, 1, ... K.\]
Then summing \eqref{thm-long-ineq} over $k=0,\ldots,K-1$, we have
\be\label{thm-con}
\begin{aligned}
F(U^0,S^0) - F* &\geq F(U^0,S^0) - F(U^K,S^K) \\
&\geq \sum_{k = 0}^{K-1}\frac{L}{2} \left(\norm{\Delta S^k}_F^2 + \norm{\Delta U^k}_F^2\right) \\
&\geq \frac{K\epsilon^2}{2L}.
\end{aligned}
\ee
Therefore, the AManPG \eqref{AManPG-general} with termination criterion \eqref{terminate-crit} finds an $\epsilon$-stationary point of problem \eqref{ourRMC-noV-rewrite} in at most $\lceil 2L( F(X^0) - F*) / \epsilon^2\rceil$ iterations.
\end{proof}


\bibliography{AManPGC}
\bibliographystyle{IEEEtran}


\end{document}